\documentclass{article}

\usepackage{PRIMEarxiv}

\usepackage[utf8]{inputenc} % allow utf-8 input
\usepackage[T1]{fontenc}    % use 8-bit T1 fonts
\usepackage{hyperref}       % hyperlinks
\usepackage{url}            % simple URL typesetting
\usepackage{booktabs}       % professional-quality tables
\usepackage{amsfonts}       % blackboard math symbols
\usepackage{nicefrac}       % compact symbols for 1/2, etc.
\usepackage{microtype}      % microtypography
\usepackage{lipsum}
\usepackage{fancyhdr}       % header
\usepackage{graphicx}       % graphics
\usepackage{xcolor}         % colors

\usepackage{multicol}
\usepackage{multirow}
\usepackage{graphicx} 
\usepackage{natbib}
\usepackage{amsmath}
\usepackage{amsmath,amsfonts,bm}

\def\eqref#1{equation~\ref{#1}}
\def\1{\bm{1}}

\def\rt{{\textnormal{t}}}

\def\ry{{\textnormal{y}}}

\def\rvc{{\mathbf{c}}}
\def\rvd{{\mathbf{d}}}

\def\rvp{{\mathbf{p}}}

\def\rvt{{\mathbf{t}}}

\def\rvy{{\mathbf{y}}}

\def\rmX{{\mathbf{X}}}

\def\rmZ{{\mathbf{Z}}}

\def\vc{{\bm{c}}}

\def\vh{{\bm{h}}}

\def\mX{{\bm{X}}}

\def\mZ{{\bm{Z}}}

\DeclareMathAlphabet{\mathsfit}{\encodingdefault}{\sfdefault}{m}{sl}
\SetMathAlphabet{\mathsfit}{bold}{\encodingdefault}{\sfdefault}{bx}{n}

\def\gE{{\mathcal{E}}}

\def\gG{{\mathcal{G}}}

\def\gV{{\mathcal{V}}}

\newcommand{\E}{\mathbb{E}}

\usepackage{amsthm}
\newtheoremstyle{defstyle}
  {\topsep} % Space above
  {\topsep} % Space below
  {} % Body font
  {} % Indent amount
  {\bfseries} % Theorem head font
  {.} % Punctuation after theorem head
  {.2em} % Space after theorem head
  {} % Theorem head spec (can be left empty, meaning `normal')

\theoremstyle{defstyle}
\newtheorem{definition}{Definition}
\newtheorem{assumption}{Assumption}

\newtheorem{proposition}{Proposition}
\theoremstyle{remark}
\newtheorem{claim}{Claim}
\theoremstyle{defstyle}

\newtheorem{problem}{Problem}
\usepackage{newfloat}
\usepackage{listings}
\usepackage{xspace}
\usepackage{bm}
\usepackage{bbold}
\usepackage{subfigure}
\graphicspath{{media/}}     % organize your images and other figures under media/ folder

\title{Learning Exposure Mapping Functions\\ for Inferring Heterogeneous Peer Effects
}

\author{
  Shishir Adhikari, Sourav Medya, Elena Zheleva \\
  Department of Computer Science \\
  University of Illinois Chicago \\
  Chicago, IL, USA\\
  \texttt{\{sadhik9, medya, ezheleva\}@uic.edu} \\
}

\newcommand{\newchange}[1]{{\textcolor{black}{{#1}}}}

\newcommand{\ourmodel}{\textsc{EgoNetGnn}\xspace}

\renewcommand{\cite}{\citep}

\begin{document}
\maketitle

\begin{abstract}
Peer effect refers to the difference in counterfactual outcomes for a unit resulting from different levels of peer exposure, the extent to which the unit is exposed to the treatments, actions, or behaviors of its peers. Peer exposure is typically captured through an explicitly defined exposure mapping function that aggregates peer treatments and outputs peer exposure. Exposure mapping functions range from simple functions like the number or fraction of treated friends to more sophisticated functions that allow for different peers to exert different degrees of influence. However, the true function is rarely known in practice and when the function is misspecified, this leads to biased causal effect estimation. To address this problem, the focus of our work is to move away from the need to explicitly define an exposure mapping function and instead introduce a framework that allows learning this function automatically. We develop \ourmodel, a graph neural network (GNN), for heterogeneous peer effect estimation that automatically learns the appropriate exposure mapping function and allows for complex peer exposure mechanisms that involve not only peer treatments but also attributes of the local neighborhood, including node, edge, and structural attributes. We theoretically and empirically show that GNN models that use peer exposure based on the number or fraction of treated peers or learn peer exposure naively face difficulty accounting for such influence mechanisms. Our evaluation on synthetic and semi-synthetic network data shows that our method is more robust to different unknown underlying influence mechanisms when compared to state-of-the-art baselines.
\end{abstract}

% keywords can be removed
% \keywords{First keyword \and Second keyword \and More}

% \vspace{-2em}
\section{Introduction}

In networked environments, the outcome of a unit can be influenced by the treatments or outcomes of other units, a phenomenon known as interference.
For example, in a contact network, the smoking habits of peers may affect an individual's respiratory health, and in a social network the political affiliations of peers may influence one's stance on a policy issue like immigration. Peer effects capture this influence by comparing an individual’s outcomes under different peer network conditions (e.g., having no smoker peers versus some smoker peers, or observed peer political affiliations versus counterfactual, flipped affiliations). Peer effect estimation is important for policy-making and targeted intervention design in many domains, including healthcare~\cite{barkley-aas20}, online advertisement~\cite{nabi-frontiers22}, and education~\cite{patacchini-jcbo17}. %\newchange{Estimating peer effects relies heavily on how we define \emph{peer exposure}, which captures the extent to which a unit is exposed to the treatments, actions, or behaviors of its peers. Most works assume that such a peer exposure mechanism is known. Our work relaxes this assumption by proposing to learn the unknown peer exposure representation automatically.}

% {Existing research has considered three types of peer exposure: binary peer exposure~\cite{bargagli-aas25}, homogeneous peer exposure~\cite{hudgens-jasa08,ugander-kdd13,jiang-cikm22,chen-icml24}, and heterogeneous peer exposure~\cite{forastiere-asa21,qu-arxiv21,yuan-www21,zhao-arxiv22}. Homogeneous peer exposure assumes all peers contribute to the exposure equally and is agnostic to the identity of the treated peers while heterogeneous peer exposure considers that different peers can exert different influence.}

\begin{figure}
    \centering
    \includegraphics[width=0.95\linewidth]{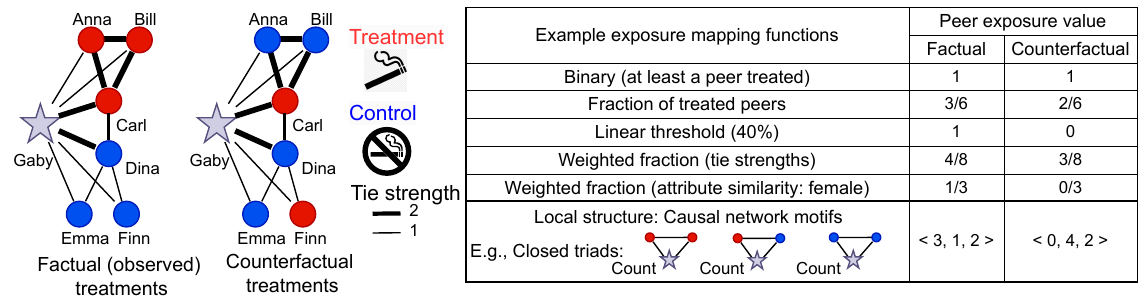}
    \caption{Illustration of different possible peer exposure representations for a node (Gaby) in a toy peer network. Red nodes represent peers in the treatment group, and blue nodes represent peers in the control group. Gray star node represents the node that has a fixed treatment.}
    \label{fig:illustration}
    \vspace{-1.5em}
\end{figure}

Peer network conditions are typically captured through an \emph{explicitly defined}
exposure mapping function~\cite{aronow-aas17} that summarizes the peer treatments and peer network and outputs peer exposure, which is the equivalent to a composite peer
treatment value. The peer effect is defined as the difference in outcomes under two distinct levels of peer exposure. Different peer exposure mapping functions capture different possible underlying influence mechanisms. Typically, domain experts define exposure
mapping functions appropriate to the causal question and domain of interest. The advantage of exposure mapping functions is that they reduce the high dimensionality of peer network attributes and that they are invariant
to irrelevant contexts (e.g., permutation of peers).  

Figure ~\ref{fig:illustration}
presents examples of prominent exposure mapping functions and the resulting peer exposure values for a toy peer network.
%\newchange{Typically, researchers manually specify an exposure mapping function~\cite{aronow-aas17} that maps high dimensional peer treatments and other relevant contexts to a low dimensional peer exposure representation.} We illustrate prominent representations of peer exposure through the toy contact network in Figure ~\ref{fig:illustration}, and include a more detailed discussion of the related work in Appendix \ref{ap:related-work}. 
%Figure ~\ref{fig:illustration} shows 
The first graph shows Gaby's peer network %i.e., a subnetwork with a unit (Gaby) and immediate peers with edges among them, 
along with the observed (i.e., factual) treatments for Gaby's peers. The second graph shows hypothetical (i.e., counterfactual) treatments for the peers. The peers in the treatment group (e.g., smokers) and control group (e.g., non-smokers) are depicted as red and blue nodes, respectively. The edge weights capture the tie strengths in the network. Binary peer exposure mapping is the simplest and it summarizes peer treatments to 0 or 1, e.g., whether any peers have been treated~\cite{bargagli-aas25} or whether the weighted treatment of peers has reached a linear threshold~\cite{tran-aaai22}. Some exposure mapping functions assume that all peers influence equally (e.g., fraction of treated peers~\cite{hudgens-jasa08,jiang-cikm22}), while others consider that different peers can exert different degrees of influence (e.g., weighted fraction~\cite{forastiere-asa21} or sum~\cite{zhao-arxiv22} of treated peers). Peer exposure has also been modeled with counts of different causal network motifs, i.e., recurrent subgraphs in a unit's peer network with treatment assignments as attributes~\cite{yuan-www21}. We discuss the related work in more detail in the Appendix \ref{ap:related-work}. %\newchange{Existing approaches fall into three broad categories: binary exposure (e.g., at least one smoking peer)~\cite{bargagli-aas25}, homogeneous exposure (e.g., fraction of smoker peers)~\cite{hudgens-jasa08,ugander-kdd13,jiang-cikm22,tran-aaai22,chen-icml24}, and heterogeneous exposure (e.g., weighted fraction based on tie strength or attributes)~\cite{forastiere-asa21,qu-arxiv21,yuan-www21,zhao-arxiv22}.} \citet{yuan-www21} represent peer exposure by counting \emph{causal network motifs}, attributed subgraphs with treatment assignments as the attributes. Figure \ref{fig:illustration} shows closed triad causal network motifs and their counts. For example, if treated peers have high clustering coefficient, Gaby could be exposed to a higher volume of secondhand smoke from her well-connected peers, assuming they smoke together. The count of the first closed triad causal network motif captures such peer exposure mechanism.

A key challenge in peer effect estimation is that the true exposure mapping function is rarely known in practice and when the function is misspecified, this leads to biased causal effect estimation. The focus of this paper is to move away from the need to explicitly define an exposure mapping function and instead learn this function automatically from data.
%Different peer exposure representations capture different possible underlying influence mechanisms. However, we rarely know the true mechanism and misspecifying the peer exposure can lead to biased causal effect estimation. \newchange{\citet{savje-biometrika24} advocates for interpretable but possibly misspecified exposure mappings and characterizes causal estimation errors due to misspecified exposure mappings, but follow-up research~\cite{auerbach-arxiv24} has highlighted the importance of capturing underlying interference mechanisms in policymaking.} Therefore, we propose learning the exposure mapping function automatically from data, which 
This has the advantage of reducing subjectivity and allowing for automated representation of peer exposure under unknown and complex peer influence mechanisms. %More specifically, we focus on learning the exposure mapping function to estimate 
More specifically, we study the problem of exposure mapping function learning in the context of heterogeneous peer effect estimation. Heterogeneous peer effects (HPE) denote variation in peer effects across individuals that may originate from personal attributes or from characteristics of their peer networks. %, where heterogeneity manifests due to the variation in %counterfactual outcomes in the units 
%peer effects for different units. 
For example, while having a friend who smokes may have a negative effect on health for some people, it may make no difference for others.   %with the same peer exposure but distinct contexts. 
% While we introduce exposure mapping function learning in the context of peer effects, the concepts can easily be adapted to other causal effects under interference, such as direct and total effects.

% \citet{ma-aistats21} summarize the covariates of treated peers using a graph neural network (GNN) to learn a peer exposure embedding in addition to homogeneous peer exposure. \citet{ma-kdd22} employ similar method but for hypergraphs to model group interactions. \cite{adhikari-mlj24} use GNNs to learn peer exposure embedding by addressing unknown peer influence mechanisms, but their scope is limited to direct effect estimation, i.e., the effect of a unit's own treatment.
% For example, \fixme{TODO add example with the same fraction of treated peers but different exposure conditions due to structure}.
{We propose \ourmodel, {a novel graph neural network (GNN) architecture}, that automatically learns a relevant exposure mapping function under appropriate identifiability assumptions. \ourmodel allows for complex peer influence mechanisms that, in addition to peer treatments, can involve the local neighborhood structure, node, and edge attributes.} Our work builds upon the success of utilizing neural networks (NNs)~\cite{shalit-icml17,im-arxiv21,shi-neurips19} and, recently, graph neural networks (GNNs)~\cite{jiang-cikm22,cai-cikm23,chen-icml24,khatami-arxiv24} for end-to-end learning of 
counterfactual outcome models or causal effect estimators.
%\textit{feature mapping functions} and \textit{counterfactual outcome models} or \textit{causal effect estimators}. 
%A feature mapping function maps raw features to feature embedding to capture potential confounders and \emph{effect modifiers}, i.e., contexts that lead to heterogeneous causal effects. A counterfactual outcome model~\cite{shalit-icml17,ma-aistats21} predicts counterfactual outcomes for different levels of treatment, while an effect estimator~\cite{shi-neurips19,chen-icml24} directly learns the causal effect of interest. 
Few studies have utilized GNNs to learn the exposure mapping function~\cite{mao-icassp25,wu-iclr25} or to derive peer exposure embedding by aggregating feature embeddings and peer treatments~\cite{adhikari-mlj24,zhao-arxiv22}. 
% \citet{mao-icassp25} have explored the use of GNNs and clustering to learn discrete exposure conditions and their probabilities, aiming to estimate overall causal effects in networks.
However, these works use off-the-shelf GNNs like GCN~\cite{kipf-iclr16} or GIN~\cite{xu-iclr18} and prior work~\cite{chen-neurips20} has shown such architectures lack expressiveness for counting subgraphs with cycles and for capturing mechanisms involving local neighborhood structure. On the other hand, counts of such subgraphs, like causal network motifs, are rich features for capturing local structural contexts~\cite{yuan-www21}, but they are expensive to compute, inflexible, and may not capture every local structural context (e.g., edge weights).

%We propose \ourmodel, \newchange{a novel graph neural network (GNN) architecture}, that automatically learns the appropriate exposure mapping function, allowing for complex peer influence mechanisms that, in addition to peer treatments, can involve the local neighborhood structure, node, and edge attributes. 
One of the biggest strengths of \ourmodel is the ability to capture the exposure mapping functions studied in previous works, including finding relevant causal network motifs and scaling to higher-order motifs. 
To add robustness to the downstream peer effect estimation task, \ourmodel is designed to learn the exposure mapping function to produce representation that is \emph{expressive} to differentiate between different peer exposure conditions and \emph{invariant} to irrelevant contexts. Moreover, \ourmodel is designed to promote bounded representation with substantial coverage of possible peer exposure values.
{Figure~\ref{fig:framework} shows an overview of \ourmodel. While most peer effect estimation frameworks contain a feature mapping and a counterfactual outcome model component, the novel additional component in ours is the custom-designed exposure mapping function learning (marked in green in the figure).}
We design this component to excel in counting attributed subgraphs, such as causal network motifs, enhancing its expressiveness to capture unknown underlying peer exposure mechanisms. % involving  neighborhood structure.
% Furthermore, \ourmodel is designed to promote invariance to irrelevant contexts and balanced representation for adding robustness to the downstream peer effect estimation task. %\newchange{}
% Recently, GNNs have been extensively used for causal effect estimation in networks~\cite{guo-wsdm20,jiang-cikm22,chen-icml24}, but their use has been mostly limited to automatic feature aggregation and addressing network confounding. 
We theoretically and empirically show that, unlike \ourmodel, existing GNN-based approaches that solely rely on homogeneous peer exposure or learn heterogeneous peer exposure naively lack expressiveness in capturing heterogeneous peer influence mechanisms based on local neighborhood structure. %Moreover, since our method is more expressive, it incorporates the exposure mapping functions studied in previous works and enables capturing relevant causal network motifs in a scalable way.
\begin{figure*}[!t]
    \centering
    \includegraphics[width=0.85\linewidth]{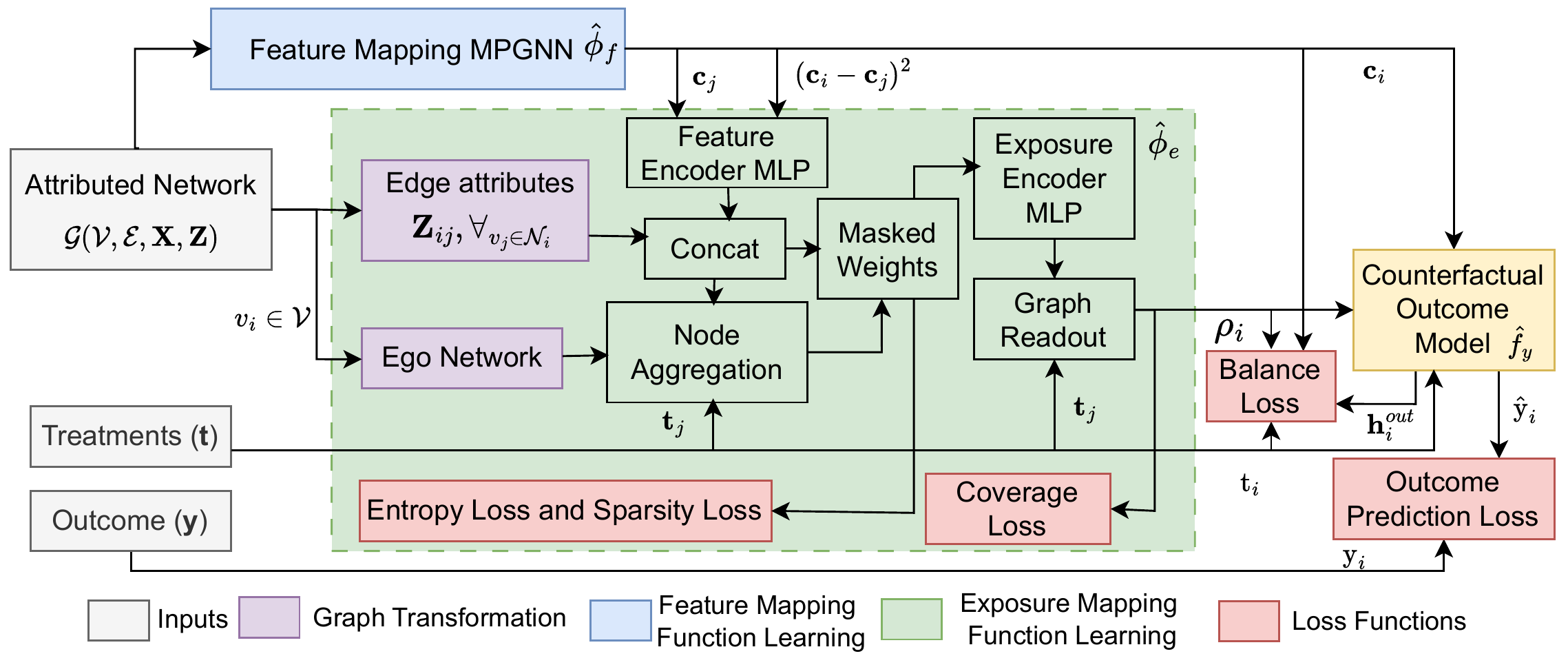}
    \vspace{-1em}
    \caption{{An overview of the proposed \ourmodel model to learn exposure mapping function for peer effect estimation. \ourmodel extracts ego networks, for each node $v_i$, with peer treatments along with feature embedding and its edge attributes as node attributes. Then, node-level aggregations are performed to capture local neighborhood contexts. These contexts are passed through a masked weight layer and encoded by an multi-layer perceptron (MLP) to learn relevant influence mechanisms and summarized with graph-level aggregation. The learned peer exposure embeddings ($\bm{\rho}_i$), along with the feature embeddings ($\rvc_i$), and treatment ($t_i$) are passed to a counterfactual outcome model that is used to infer peer effects. The graph transformation ensures expressiveness, while balance, coverage, entropy, and sparsity losses promote the robustness of the peer exposure representation.}}
    \label{fig:framework}
    \vspace{-1em}
\end{figure*}

% \vspace{-1em}
\section{\newchange{Causal Inference Problem Setup}}
\label{sec:problem_setup}
\textbf{Notations}. We represent the network as an undirected graph {\small $\gG=({\gV},{\gE})$} with a set of {\small $n=|{\gV}|$} nodes, a set of edges {\small ${\gE}$}, node attributes {\small $\rmX$}, and edge attributes  {\small $\rmZ$}.
Let {\small $\rvt=<\rt_1,...,\rt_i,...,\rt_n>$} be a random variable comprising the treatment variables {\small $\rt_i$} for each node $v_i \in \gV$ in the network and {$\ry_i$} be a random variable for $v_i$'s outcome.
Let {\small $\bm{\pi}=<\pi_1,...,\pi_i,...,\pi_n>$} be an assignment to {\small $\rvt$} with {\small $\pi_i \in \{0,1\}$} assigned to {\small $\rt_i$}. 
Let {\small $\rvt_{-i}=\rvt \setminus \rt_i$} and {\small $\bm{\pi}_{-i}=\bm{\pi} \setminus \pi_i$} denote random variable and its value for treatment assignment to other units except $v_i$, {which we refer to as peer treatments for convenience}.

% \newchange{\textbf{Exposure mapping function}. Peer effect, in Eq. \ref{eq:peer_eff_intermediate}, is defined in terms of peer exposure, but peer exposure itself is not observed directly and cannot be intervened upon. Peer exposure, $\bm{\rho}_i$, depends on peer treatments, $\mathbf{T}_{-i}=\bm{\pi}_{-i}$, and other relevant contexts in the set $\{G, \mathbf{X}, \mathbf{Z}\}$), which are determined by unknown underlying influence mechanisms.
{Peer exposure reflects how much a unit is exposed to peer treatments and is defined as follows.}
\begin{definition}[Peer exposure and exposure mapping function]\label{def:peer-exposure}
    Peer exposure for unit $v_i$ is defined as {\small $\bm{\rho}_i \in [0,1]^d = \phi_e(\bm{\pi}_{-i}, \gG, \rmX, \rmZ)$}, where $\phi_e$ is the \textit{exposure mapping function} that maps high-dimensional contexts {\small $\{ \bm{\pi}_{-i}, \gG, \rmX, \rmZ\}$} to a $d$-dimensional peer exposure representation bounded between $0$ and $1$ such that {\small $\ry_i(\rt_i=\pi_i, \rvt_{-i}=\bm{\pi}_{-i})|\{\gG,\rmX, \rmZ\} = \ry_i(\rt_i=\pi_i, \rvp_{i}=\bm{\rho}_i)|\{\gG,\rmX, \rmZ\}$}.
\end{definition}
{Definition \ref{def:peer-exposure} maps peer treatments {\small $\rvt_{-i}=\bm{\pi}_{-i}$} and peer network contexts {\small $\{\gG,\rmX, \rmZ\}$}} to peer exposure {\small $\rvp_{i}=\bm{\rho}_i$} in terms of equivalence of counterfactual outcomes {\small $\ry_i(\rt_i=\pi_i, \rvt_{-i}=\bm{\pi}_{-i})$ and  $\ry_i(\rt_i=\pi_i, \rvp_{i}=\bm{\rho}_i)$}. Here, {\small $\ry_i(\rt_i=\pi_i, \rvp_{i}=\bm{\rho}_{i})$}, captures that, in interference settings, {the counterfactual outcome of a unit $v_i$ is influenced by both unit's treatment {\small $\rt_i=\pi_i$} and peer exposure {\small $\rvp_{i}=\bm{\rho}_{i}$}}.
Note that the exposure mapping function could map different contexts to the same peer exposure. 
% In this work, we focus on learning the exposure mapping function $\phi_e$ for estimating heterogeneous peer effects defined next.

% \textbf{Heterogeneous peer effect}. 
% Three main types of causal effects are studied in the context of interference: direct effects, peer effects and total effects. 
{Peer effect refers to the difference in counterfactual outcomes for different values of peer exposure.} %The \textit{heterogeneous peer effect} (HPE) for a unit $v_i$ for peer exposures {\small $\rvp_{i}=\bm{\rho}_{i}$} versus {\small $\rvp_{i}=\bm{\rho}'_{i}$} and unit's treatment {\small $\rt_i=\pi_i$} conditioned on the unit's contexts { $\rvc_i$} is defined as:
Heterogeneous peer effects (HPE) refers to different units having different peer effects dependent on their contexts. For any given unit $v_i$, its heterogeneous peer effect is described through its context, i.e., for peer exposures {\small $\rvp_{i}=\bm{\rho}_{i}$} versus {\small $\rvp_{i}=\bm{\rho}'_{i}$} and unit's treatment {\small $\rt_i=\pi_i$} conditioned on the unit's contexts { $\rvc_i$}, it is defined as:
{\small
\begin{equation}
    \label{eq:peer_eff_intermediate}
    \delta_i(\bm{\rho}_{i}, \bm{\rho}'_{i}) = \mathbb{E}[\ry_i(\rt_i=\pi_i, \rvp_{i}=\bm{\rho}_{i})| \rvc_i] - \mathbb{E}[\ry_i(\rt_i=\pi_i, \rvp_{i}=\bm{\rho'}_{i}) | \rvc_i],
\end{equation}
}% The conditioning of {\small $\rvc_i$} in Eq. \ref{eq:peer_eff_intermediate} indicates that the counterfactual outcome for the same value of treatment $\rt_i$ and peer exposure $\rvp_i$ could be heterogeneous and vary for different unit $v_i$ depending on context {\small $\rvc_i$}, referred to as \textit{effect modifiers} (e.g., unit's degree or node attribute).
\newchange{where expectation is over units with similar contexts {\small $\rvc_i$}, referred to as \textit{effect modifiers} (e.g., unit's degree or node attributes), defined by a feature mapping function of contexts {\small $\{\gG,\rmX, \rmZ\}$} from $v_i$'s perspective, i.e., {\small $\rvc_i=\phi_f(v_i, \gG, \rmX, \rmZ)$}.}
% The effect modifiers {\small $\rvc_i$} are defined by a feature mapping function of node attributes {\small $\rmX$}, edge attributes {\small $\rmZ$}, and network structure {\small $\gG$}, i.e., {\small $\rvc_i=\phi_f(\gG, \rmX, \rmZ)$}.
% Our work focuses on learning the exposure mapping function {\small $\phi_e$} for estimating heterogeneous peer effects.
Substituting peer exposures {\small $\bm{\rho}_{i}$} and {\small $\bm{\rho}'_{i}$} with corresponding exposure mapping functions for two peer treatment assignments $\bm{\pi}_{-i}$ versus $\bm{\pi'}_{-i}$ in Eq. \ref{eq:peer_eff_intermediate}, we get:
{\small
\begin{equation}\label{eq:peer_eff_exp_map}
    % \begin{split}
    \delta_i(\bm{\pi}_{-i}, \bm{\pi}'_{-i}) = \mathbb{E}[\ry_i(\rt_i=\pi_i, \rvp_{i}=\phi_e(\bm{\pi}_{-i}, \gG, \rmX, \rmZ))| \rvc_i] - \mathbb{E}[\ry_i(\rt_i=\pi_i, \rvp_{i}=\phi_e(\bm{\pi}'_{-i}, \gG, \rmX, \rmZ)) | \rvc_i].
    % \end{split}
\end{equation}
}
\textbf{Causal identification}. Now, we discuss the identification of peer effects that involves expressing counterfactual outcomes in terms of observational and/or interventional distributions.
% : (1) pre-treatment network and attributes and (2) neighborhood interference.}

Next, we make two commonly adopted assumptions in network interference settings.
\begin{assumption}[Pre-treatment network]\label{assum:pre}
The network {\small $\gG$} along with node attributes {\small $\rmX$} and edge attributes {\small $\rmZ$} are measured before treatment assignments {\small $\rvt=\bm{\pi}$} and treatments are not mutable.    
\end{assumption}
\begin{assumption}[Neighborhood Interference]\label{assum:neigh-int}
The counterfactual outcome of a unit depends only on its immediate neighborhood treatments, i.e., {\small $\ry_i(\rt_i=\pi_i, \rvt_{-i}=\bm{\pi}_{-i})|\rvc_i = \ry_i(\rt_i=\pi_i, \rvt_{-i}^{\mathcal{N}_i}=\bm{\pi}_{-i}^{\mathcal{N}_i})|\rvc_i=\ry_i(\rt_i=\pi_i, \rvp_{i}=\phi_e(\bm{\pi}_{-i}^{\mathcal{N}_i}, \gG, \rmX, \rmZ))|\rvc_i$}, where {\small $\mathcal{N}_i=\{j:(v_i,v_j) \in \gE\}$}, {\small $\rvt_{-i}^{\mathcal{N}_i}=\rvt_{-i} \cap \{\rt_j: j \in \mathcal{N}_i\}$}, and {\small $\bm{\pi}_{-i}^{\mathcal{N}_i} = \bm{\pi}_{-i} \cap \{\pi_j: j \in \mathcal{N}_i\}$} denote neighborhood set, treatments, and assignments, respectively.
\end{assumption}
{Assumption \ref{assum:pre} is a general assumption in experimental and observational studies, 
% like A/B tests and prospective cohort studies in networks to facilitate identification of causal effects.
and Assumption \ref{assum:neigh-int} is a common simplifying assumption that presumes network influence is mediated by immediate neighbors but our work could be extended to consider interference from multiple-hop neighborhoods. For ease of exposition, we drop the superscript $\mathcal{N}_i$ in neighborhood treatments and assignments.}

For causal identification, we assume unconfoundedness, similar to previous work~\cite{ma-kdd22,wu-iclr25}:
\begin{assumption}[Unconfoundedness]\label{asum:unconfoundedness}
    \newchange{For all unit treatment {\small $\pi_i\in \{0,1\}$} and peer treatment assignments {\small $\bm{\pi}_{-i}\in\{0,1\}^{n-1}$}, there exists a feature mapping function {\small $\phi_f \in \Phi_{f}$} and an exposure mapping function {\small $\phi_e \in \Phi_{e}$} such that the counterfactual outcome is independent of unit treatment and peer exposure conditions given the context \small {$\rvc_i=\phi_f(v_i, \gG, \rmX, \rmZ)$, i.e., $\ry_i(\rt_i=\pi_i, \rvp_{i}=\phi_e(\bm{\pi}_{-i}, \gG, \rmX, \rmZ)) \perp \{\rt_i, \rvp_{i}\} | \rvc_i$}.}
\end{assumption}
\newchange{Assumption \ref{asum:unconfoundedness} implies that the observed network context is sufficient for controlling for confounding, and there are functions able to represent it compactly. Under this assumption, it is still possible to learn a feature mapping and exposure mapping functions that do not approximate the true functions which leads to a misspecification error. Therefore, it is important to learn an expressive function (e.g., a GNN) that is able to capture a wide range of possible functions. %The first condition is untestable, but the second can be achieved with expressive GNNs. Misspecification error occurs when learned functions fail to approximate the true ones.
We also assume the standard \textit{consistency} (Assumption \ref{asum:consistency}) and \textit{positivity} (Assumption \ref{asum:pos}), described in more detail in Appendix \ref{ap-problem-setup}.
Next, we present the causal identification conditions and formally define the problem of exposure mapping function learning in the context of peer effect estimation.}

{
\begin{proposition}\label{prop:estimation}
    With Assumptions \ref{assum:pre}-\ref{asum:pos}, the HPE $\delta_i$ in Eq. \ref{eq:peer_eff_exp_map} can be estimated from experimental or observational data as 
    {\small
\begin{equation}\label{eq:peer_exp_obs_main}
    \delta_i(\bm{\pi}_{-i}, \bm{\pi}'_{-i}) =\mathbb{E}[\ry_i|\rt_i=\pi_i, \rvp_{i}=\phi_e(\bm{\pi}_{-i}, \gG, \rmX, \rmZ), \rvc_i] -\mathbb{E}[\ry_i| \rt_i=\pi_i, \rvp_{i}=\phi_e(\bm{\pi'}_{-i}, \gG, \rmX, \rmZ),\rvc_i].
\end{equation}}
\end{proposition}}
The proof presented in Appendix \ref{ap-problem-setup} stems from consistency and unconfoundedness assumptions.
% {After applying standard causal inference assumptions of \textit{consistency}, \textit{positivity}, and \textit{unconfoundedness}, the HPE $\delta_i$ in Eq. \ref{eq:peer_eff_exp_map} can be estimated from experimental or observational data as follows.

% We defer the detailed explanation of \textit{consistency}, \textit{positivity}, and \textit{unconfoundedness} assumptions and derivation of Eq. \ref{eq:peer_exp_obs_main} to Appendix \ref{ap-problem-setup} due to limited space.}

% \textbf{Learning and Estimation}. Peer effects can be estimated using the network structure, node attributes, edge attributes, and treatments as inputs to learn two functions $\phi_f$ for contexts $\mathbf{C}_i$ and $\phi_e$ for peer exposure $\mathrm{P}_i$ and estimate two conditional expectations of the outcome. 
% \fixme{TODO: positivity and consistency assumption in main text and discuss how well these assumptions are satisfied??}Formally, the problem of exposure mapping function learning is defined in the context of peer effect estimation as follows.
{\begin{problem}[Exposure mapping function learning]
{
    Given network contexts {\small $\{\gG, \rmX, \rmZ\}$}, treatments $\rvt$, and outcomes $\rvy$ of $n$ units, estimate the feature and exposure mapping functions $\hat{\phi_f}$ and $\hat{\phi_e}$ along with counterfactual outcome model $\hat{f}_{y}$ such that mean squared error between true heterogeneous peer effect (HPE) $\delta_i$ and estimated HPE  $\hat{\delta_i}$, i.e., {\small  $\frac{1}{n}\sum_{i=1}^n(\delta_i - \hat{\delta_i})^2$}, is minimized, 
    % {\small
    % \begin{equation}
    %     \frac{1}{N}\sum_{i=1}^N(\tau_i - \hat{\tau_i})^2,
    % \end{equation}
    % }
    where {\small $\hat{\delta_i}=\hat{f_{y}}(\pi_i, \hat{\bm{\rho}}_i, \hat{\rvc}_i) - \hat{f}_{y}(\pi_i, \hat{\bm{\rho}}'_i, \hat{\rvc}_i)$} with {\small $\hat{\bm{\rho}}_i = \hat{\phi_e}(\bm{\pi}_{-i}, \gG, \rmX, \rmZ)$}, {\small $\hat{\bm{\rho}}'_i = \hat{\phi_e}(\bm{\pi}'_{-i}, \gG, \rmX, \rmZ)$}, and {\small $\hat{\rvc}_i=\hat{\phi}_f(v_i, \gG, \rmX, \rmZ)$}.}
\end{problem}
{The true HPE is unknown, but due Proposition \ref{prop:estimation}, the factual outcomes can be utilized to jointly estimate $\hat{\phi_f}$, $\hat{\phi_e}$, and  $\hat{f}_{Y_i}$ as discussed in the next section.}}

\section{\ourmodel: Learning Exposure Mapping Function with GNNs}

{Figure \ref{fig:framework} shows an overview of the proposed \ourmodel model to simultaneously learn exposure mapping function $\hat{\phi_e}$, feature mapping function $\hat{\phi_f}$, and counterfactual outcome model $\hat{f}_{y}$ for peer effect estimation.}
% \textit{First,} the attributed network $G(V, E, \mathbf{X}, \mathbf{Z})$ is passed through a message passing graph neural network (MPGNN) that approximates feature mapping function $\hat{\phi_f}$ to learn feature embedding $\mathbf{C}_i$ that captures confounders or effect modifiers. \textit{Second}, \ourmodel performs graph transformation to extract an ego network, for each node $v_i$, with peer treatments along with node $v_i$'s edge attributes and feature embeddings as node attributes. Then, node-level aggregations are performed to capture local neighborhood contexts. These contexts are passed through a masked weight layer and encoded by an multi-layer perceptron (MLP) to learn relevant influence mechanisms and summarized with graph-level aggregation with a graph readout function. The learned peer exposure embeddings ($\bm{\rho}_i$), along with the feature embeddings ($\mathbf{C}_i$), treatment ($T_i$), and outcomes ($Y_i$), are passed to a train a counterfactual outcome model.
We aim to learn exposure mapping function $\hat{\phi}_e$ with three key properties: 1) expressiveness, 2) invariance, and 3) bounded and balanced representation. The expressiveness property ensures the peer exposure representation $\bm{\rho}_i$ returned by the function $\hat{\phi}_e$ is unique for different relevant contexts, while the invariance property assures the representation $\bm{\rho}_i$ does not vary due to irrelevant contexts. \newchange{For example, in Figure \ref{fig:illustration}, if the underlying peer influence depends on clustering coefficients among treated, the function $\hat{\phi}_e$ is expressive if it can capture the first closed triad substructure.} \newchange{The standard message passing GNNs (e.g., GCN, GIN, etc) cannot capture essential causal network motifs like closed triads (i.e., triangular motifs)~\cite{chen-neurips20}. The graph transformation and automated exposure mapping function learning in our \ourmodel model are designed to ensure that the peer exposure representation is at least as expressive as or superior to the approach of feature extraction by counting causal network motifs.} In the above example, the function $\hat{\phi}_e$ is invariant to irrelevant contexts if the difference in other features like node attributes and edge weights do not change the learned representation $\bm{\rho}_i$. 
To satisfy the property of bounded representation, the learned representation $\bm{\rho}_i$ should be bounded, e.g., between 0 and 1, to reflect no exposure and maximum exposure. \newchange{Moreover, the representation $\bm{\rho}_i$ should have a substantial coverage, which means it should be distributed across the possible range of exposure. Next, we describe our feature mapping, exposure mapping, and counterfactual outcome model in detail.}

\subsection{Architecture of \ourmodel}
\ourmodel first maps the attributed network to feature embedding using a MPGNN and extracts ego networks for each node $v_i$, incorporating peer treatments, node features, and edge attributes. It performs node-level aggregation to capture local context, which is processed through a masked weight layer and an MLP followed by graph-level aggregation to learn peer exposure representation.

\textbf{Feature mapping MPGNN}.
{The feature mapping module aims to capture contexts that are potentially confounders or effect modifiers. 
% Capturing confounders ensures the estimates are unbiased and valid, while capturing effect modifiers reduces error in unit-level causal effect estimates. GNNs are shown to be suitable for capturing such contexts in network settings~\cite{guo-wsdm20,jiang-cikm22,adhikari-mlj24}. 
Let $\Theta$ denote a multi-layer perceptron (MLP) and {\small $||$} denote a concatenation operator. The feature embedding $\vc_i$ is obtained for $l$-th layer as:
{\small
\begin{equation}\label{eq:fmap}
    \vc_i = \Theta_0(\mX_i) || \vh_i^l\text{ and }\vh_i^l = \vh_i^{l-1} + \sum_{j \in \mathcal{N}_i}\Theta_{l}\vh^{l-1}_j,
\end{equation}
}where {\small $\vh^0_j = \mX_j || \mZ_{ij}$}, and {\small $\vh^0_i=0$} are initial conditions and {\small $\mathcal{N}_i$} denote neighbors of node $v_i$. This MPGNN architecture incorporates edge attributes {\small $\mZ_{ij}$} while disentangling the hidden representation of the unit's own attributes {\small $\Theta_0(\mX_i)$} from that of aggregated peer and edge attributes {\small $\vh_i^l$}.}

\textbf{Ego network construction}. {To learn an exposure mapping function that is as least as expressive as or superior to the approach of feature extraction by counting network motifs, we transform the node regression task to graph regression by extracting ego networks for each unit. In an ego network, the triangle structures involving an ego node are transformed as edges, which mitigates the limitation of GNNs to capture closed triad motifs.}
The ego network {\small $\Bar{\gG_i}(\Bar{\gV_i}, \Bar{\gE_i})$} is extracted from {\small $\gG(\gV,\gE)$} for each node $v_i$ such that node set {\small $\Bar{\gV_i}$} consists neighbors of $v_i$, i.e., {\small $\Bar{\gV_i} = \{v_j: e_{ij} \in \gE \wedge v_j \in \gV\}$} and edge set {\small $\Bar{\gE_i}$} consists edges between neighbors of $v_i$, i.e., {\small $\Bar{\gE_i} = \{e_{jk} : e_{jk} \in \gE \wedge v_j \in \Bar{\gV_i} \wedge v_k \in \Bar{\gV_i}\}$}.

{\textbf{Feature encoder and node aggregation}. Feature encoder module takes relevant peer feature embeddings and the distance between ego and peer feature embeddings, i.e., {\small $\vc_{ij} = \Theta_{feat}(\vc_j || (\vc_i - \vc_j)^2)$}, to capture peer influence mechanisms involving peer attributes and feature similarity between ego and peers.}
{Then, we transform an ego $v_i$'s edge attributes {\small $\mZ_{ij}$} to node attributes, i.e., {\small $\Bar{\mX}_{j} = \mZ_{ij}$}, in the ego network {\small $\Bar{\gG_i}(\Bar{\gV_i}, \Bar{\gE_i})$} because the ego $v_i$ itself is not present in the ego network. 
% Next, node attribute $\Bar{X}_j$ of node $v_j \in \Bar{V_i}$ in the ego network is set using edge attributes of ego node $v_i$ and peer $v_j$, i.e., $\Bar{X_{j}} = Z_{ij}$.
The node aggregation for each node $v_j$ in the ego network {\small $\Bar{\gG_i}$} considers neighbors' node attributes $\Bar{\mX}_k$, feature encoding $\rvc_{ik}$, edge attributes $\mZ_{jk}$, and peer treatments $t_k$, and is defined  for $l^{th}$ layer as follows:
{\small
\begin{equation}\label{eq:node-agg}
    \vh^l_j= \vh^{l-1}_j + \sum_{k \in \mathcal{N}_j} \vh^{l-1}_k, \text{ with }\vh^0_k = t_k || \Bar{\mX}_k || \vc_{ik} || \mZ_{jk}  \text{ and } \vh^0_j=0.
\end{equation}}
}\textbf{Masked weights and exposure encoder}. Masked weights promotes representation that is invariant to irrelevant contexts and feeds the concatenation of node attributes and hidden state after $L$ layers of node aggregation, i.e., {\small $\vh_j^{agg} = \Bar{\mathbf{X}}_j || \vc_{ij} || \vh_j^L$}, through a \textit{masked fully connected} layer as follows:
{\small
\begin{equation}
    \vh_j^{mask} = ReLU((\sigma(\mathbf{W}_{mask})\odot\mathbf{W}_{agg})\vh_j^{agg} + \mathbf{b}_{agg}),
\end{equation}
}
where {\small $ReLU$} and {\small $\sigma$} are a rectified linear unit and sigmoid activation functions, $\odot$ indicates element-wise product, {\small $\mathbf{W}_{mask}$ and $\mathbf{W}_{agg}$} are the weight matrices, and {\small $\mathbf{b}_{agg}$} is the bias vector.
{The masked hidden representation $\vh^{mask}_j$ is passed into an exposure encoder MLP to extract a low dimensional embedding. The goal of this module is to capture complex mechanisms based on the local neighborhood and reduce dimensionality. Formally, the output embedding {\small $\vh^{exp}_j$} is obtained as follows:
{\small
\begin{equation}
    \vh^{exp}_j = ReLU(\Theta_{exp}(ln(ReLU(\Theta_{enc}(\vh^{mask}_j))+1))),
\end{equation}
}
{\small $\Theta_{enc}$} and {\small $\Theta_{exp}$} are two MLPs and $ln$ denotes log transformation that offers the benefit of rescaling features with large values that are significant in scale-free networks (e.g., online social networks) and introduces inductive bias to capture mechanisms involving ratios.}
% Here, the intermediate layer uses $tanh$ activation function to identify mechanisms that may involve proportions (i.e., multiplication or division) and $tanh$ helps the subsequent MLP to learn it by bounding the input.

\textbf{Graph readout}.
Finally, the peer exposure embedding $\bm{\rho}_i$ for node $v_i$ is obtained by aggregating the representation $\vh_j^{exp}$ for all $v_j \in \Bar{\gV_i}$ on the entire ego network {\small $\Bar{\gG_i}(\Bar{\gV_i}, \Bar{\gE_i})$} as 
% {\small
% \begin{equation}
   {\small $\bm{\rho}_i={\sum_j(t_j \times \vh^{exp}_j)}/{\sum_j \vh^{exp}_j} || 1 - e^{-\sum_j(t_j \times \vh^{exp}_j)}$}. 
% \end{equation}
% }
We consider two aggregations such that the peer exposure embedding is bounded between $0$ and $1$, with $0$ being the case of no peer exposure. The first aggregation \newchange{captures proportion} similar to the fraction of treated peers, but we weight each peer by ${\vh^{exp}_j}/{\sum_j \vh^{exp}_j}$ learned by the preceding layer. The second aggregation \newchange{captures scale and} is analogous to the number of treated peers, except that each peer is weighted by $\vh^{exp}_j$.

% \fixme{TODO: TLearner connection

\subsection{End-to-end Learning of \ourmodel}
% Here we describe the end-to-end learning framework of \ourmodel.
\newchange{The resulting peer exposure embeddings  ($\bm{\rho}_i$) and the feature embeddings ($\vc_i$) from the above module along with unit treatment ($\pi_i$) are passed to a counterfactual outcome model $f_y(\rt_i=\pi_i, \rvp_i=\bm{\rho}_i, \rvc_i=\vc_i)$ to obtain conditional counterfactual outcome $\E[\ry_i(\rt_i=\pi_i,\rvp_i=\bm{\rho}_i)|\rvc_i=\vc_i]=\E[\ry_i|\rt_i=\pi_i,\rvp_i=\bm{\rho}_i,\rvc_i=\vc_i]$ (Eq. \ref{eq:peer_exp_obs_main}). We adapt the Treatment Agnostic Representation Network (TARNet) and Counterfactual Regression (CFR) models~\cite{shalit-icml17} as the counterfactual outcome model $\hat{f_y}$. The TARNet architecture consists of a single embedding MLP and two prediction heads to estimate counterfactual outcomes with unit treatment $\rt_i=1$ and $\rt_i=0$, i.e.,
{\small
\begin{equation}
    \vh^{emb}_i = \Theta_{emb}(\vc_i)||\bm{\rho}_i,~~~~ \hat{\ry}_i(0)=\Theta_{y_0}(\vh^{emb}_i),~~~~ \hat{\ry}_i(1)=\Theta_{y_1}(\vh^{emb}_i).
\end{equation}
} The CFR architecture is similar except for an autoencoder to produce the embeddings, i.e., {\small $\vh^{emb}_i = \Theta_{emb}(\vc_i||\bm{\rho}_i)$} and {\small $\vh^{out}_i=\Theta_{dec}(\vh_i^{emb})$}.
The CFR or TARNet model {\small $\hat{f}_{y}(\pi_i, \bm{\rho}_i, \vc_i)$} predicts outcome {\small $\hat{y}_i = \hat{y}_i(1)$} if {\small $\pi_i=1$} and {\small $\hat{y}_i = \hat{y}_i(0)$} if {\small $\pi_i=0$}. The unit-level factual prediction loss $\mathcal{L}_{y_i}$ is defined as
{\small
\begin{equation}
\begin{split}
        \mathcal{L}_{y_i} = loss(y_i, \hat{f}_{y}(\rt_i=\pi_i, \hat{\rvp_i}=\hat{\phi_e}(\bm{\pi}_{-i}, \gG, \mX, \mZ;\Theta_e), \hat{\rvc_i}=\hat{\phi_f}(v_i, \gG,\mX, \mZ;\Theta_f);\Theta_y)), 
        % \forall_{v_i \in V}, T_i=\pi_i \wedge T_{-i}=\bm{\pi}_{-i},
\end{split}
\end{equation}
}where $loss$ is an appropriate loss function (e.g., square error loss) based on data type of the outcome and $\bm{\Theta} = \{\Theta_e, \Theta_f, \Theta_y\}$ are learning parameters to be optimized for exposure mapping function $\hat{\phi_e}$, feature mapping function $\hat{\phi_f}$, and counterfactual outcome model $\hat{f}_{y}$, respectively.}

\newchange{\textbf{Balance loss}. The CFR architecture uses autoencoder reconstruction loss and the Integral Probability Metric (IPM)~\cite{shalit-icml17} measure of distance between treatment and control groups using Wasserstein~\cite{cuturi-icml14,arjovsky-icml17}, jointly referred to as balance loss, i.e.,
{\small
\begin{equation}
    \mathcal{L}_{bal} = \1_{\lambda_{bal} > 0} \times \frac{1}{n}\textstyle \sum_i(\vh^{out}_i - \vc_i||\bm{\rho}_i)^2 + \lambda_{bal} \times IPM(\{\vh^{emb}_i: t_i=1\}, \{\vh^{emb}_i: t_i=0\}),
\end{equation}
}
where $\lambda_{bal} \ge 0$ is a hyperparameter and $IPM(.)$ balances the distribution $\mathbb{P}(\rvc,\rvp|\rt=0)$ and $\mathbb{P}(\rvc,\rvp|\rt=0)$, where $\mathbb{P}(\rvc,\rvp|\rt)$ is equivalent to $\mathbb{P}(\rvp|\rt)\mathbb{P}(\rvc|\rvp,\rt)$. Intuitively, $\mathcal{L}_{bal}$ balances peer exposure distribution $\rvp$ between treatment groups and covariate distribution $\rvc$ across peer exposure conditions and treatment groups while maintaining expressiveness due to the autoencoder component.}

{For the end-to-end learning of $\hat{\phi}_e$,  $\hat{\phi}_f$, and $\hat{f}_{Y_i}$, we introduce three custom loss functions designed for \ourmodel: \newchange{coverage loss}, sparsity loss, and entropy loss. These custom loss functions serve as priors to make the learned exposure mapping function stable and reliable.}

% \textbf{TARNet outcome prediction loss}. This loss function minimizes the MSE error between predicted outcome and observed outcome, i.e.,
% $L_{pred} = (Y_i - \hat{Y_i})^2$, where $\hat{Y_i} = Y(1, \bm{\rho}_i) \text{ if } T_i=1 \text{ else } Y(0, \bm{\rho}_i)$.

\newchange{\textbf{Coverage loss}.} We use a prior that encourages \newchange{the bounded peer exposure embedding to have substantial coverage.} This loss function checks how far the learned peer embedding distribution is from a continuous uniform distribution between $0$ and $1$, i.e.,
$
L_{cov} = (mean(\bm{\rho}) - 0.5)^2 + (var(\bm{\rho}) - \frac{1}{12})^2 + (range(\bm{\rho}) - 1)^2$. Here, we consider mean squared error of mean, variance, and range of learned embedding $\bm{\rho}$ against corresponding value of the uniform distribution.

{\textbf{Entropy loss and sparsity loss}. Entropy loss encourages mask weights, i.e., $p:=\sigma(\mathbf{W}_{mask})$ to take values toward $0$ or $1$ and sparsity loss pushes for a few weights with high values. Formally, we define entropy loss and sparsity loss as $\mathcal{L}_{ent}=mean(-p log(p) - (1-p) log(1-p))$ and $\mathcal{L}_{sp}=mean(p)$.}
% \textbf{Bound loss}. For reliability, we use a prior that peer effects for the instances with no exposure are zero. This loss function checks if peer effects for the instances with $\bm{\rho}_i=0$ are zero, i.e., $L_{bound} = (Y_i(\pi_i,\bm{\rho}_i) - Y_i(\pi_i,\bm{\Vec{0}}))^2$ if $\bm{\rho}_i=0$ else $0$. This is required for the reliability of the \ourmodel framework and for preserving the interpretation that $\bm{\rho}_i=0$ means no peer exposure. Notice the second term $Y_i(\pi_i,\bm{\Vec{0}})$ represents a counterfactual setting with no peer exposure for all units and it is a significant distribution shift from the observed peer exposure conditions. This loss function aims to mitigate the effect of distribution shifts.

\textbf{Overall loss}. {We obtain the overall loss function $\mathcal{L}$ as
\begin{equation}
    \mathcal{L}=\frac{1}{n}\textstyle \sum_i\mathcal{L}_{y_i} + \mathcal{L}_{bal}+ \lambda_{cov} \times \mathcal{L}_{cov} + \lambda_{ent} \times \mathcal{L}_{ent} + \lambda_{sp} \times \mathcal{L}_{sp} + \lambda_{L1} \times ||\bm{\Theta}_{gnn}||_1,
\end{equation}
where $\lambda_{cov}$, $\lambda_{ent}$, and $\lambda_{sp}$ are the hyperparameters and $\bm{\Theta}_{gnn}$ denote overall parameters in $\hat{\phi}_f$ and $\hat{\phi}_e$, and the last term is $L_1$ loss to promote invariance to irrelevant contexts by preferring sparse weights.}

\textbf{Inference}. The peer effect is obtained as {\small $\hat{\delta_i}(\bm{\pi}_{-i}, \bm{\pi}'_{-i}) = \hat{f}(\pi_i, \bm{\pi}_{-i}, \gG,\rmX, \rmZ) - \hat{f}(\pi_i, \bm{\pi}'_{-i}, \gG,\rmX, \rmZ) = \hat{f}_{y}(\pi_i, \bm{\rho}_i, \rvc_i) - \hat{f}_{y}(\pi_i, \bm{\rho}'_i, \rvc_i)$}, where $\hat{f}$ is the end-to-end \ourmodel.

\vspace{-1em}
\subsection{Theoretical Analyses of \ourmodel}

\textbf{Expressiveness}. We perform a theoretical analysis of the expressive power of graph neural networks (GNNs) in capturing the causal network motifs proposed in the \citet{yuan-www21} paper. Building on previous research regarding the capacity of GNNs to count substructures~\cite{chen-neurips20}, we demonstrate that existing message-passing GNN methods are not expressive enough to capture all causal network motifs. In contrast, our method is expressive to capture relevant causal network motifs. We defer the detailed theoretical framework and results to Appendix \ref{ap-theory}.

\textbf{Time complexity}. 
% The increased expressiveness of our model comes with the trade-off of increased runtime to process ego networks.
Our analysis of runtime complexity included in Appendix \ref{ap-theory-time} shows our method is, roughly on average, $\rho_{\gE} \times avg(\rvd)$ times more computationally expensive than standard MPGNNs, where $\rho_{\gE}$ is the average edge density and $avg(\rvd)$ is the average degree.

\newchange{\textbf{Misspecification errors}. We extend \citet{shalit-icml17}'s analyses of theoretical counterfactual prediction error bounds for the CFR model to study misspecification errors in the end-to-end \ourmodel using the sequential error decomposition trick in {Appendix \ref{ap-theory-error}.} By focusing on learning the expressive exposure mapping function, we are reducing its misspecification error directly.}

\vspace{-1em}
\section{Experiments and Results}
% Here, we first describe the datasets and experimental setup for the evaluation of \ourmodel. Then, we present the main takeaways from the results.
% \subsection{Dataset}
\vspace{-1em}
\subsection{Experimental Setup}
\label{sec:exp_setup}

\textbf{Dataset}. Similar to other works in causal inference, we rely on synthetic and semi-synthetic data. We consider three synthetic network models with a fixed number of nodes ($N=3000$) with different data generating parameters and edge densities: (1) the Watts Strogatz (WS) network~\cite{watts-nature98}, which models small-world phenomena, (2) the Barab{\'a}si Albert (BA) network~\cite{albert-rmp02}, which models preferential attachment phenomena, and (3) the Stochastic Block Model (SBM) that models community structures. We control the density of edges for BA and WS networks and the number of communities in the SBM network.
% For the BA model, the preferential attachment parameter $m \in [1, 5, 10]$ is used to generate sparse to dense networks, where a new node connects to $m$ existing nodes to form the network. For the WS model, we set mean degree parameters $k \in \{0.002N, 0.005N, 0.01N\}$ with fixed rewiring probability of $0.5$, similar to prior works~\cite{yuan-www21,adhikari-mlj24}. For the SBM model, we use number of blocks parameters $b \in \{500, 200, 100\}$ with randomly generated edge probabilities within and across communities. 
We also use a real-world social networks BlogCatalog and Flickr with more realistic topology and attributes to generate treatments and outcomes. We defer additional details on data generation to Appendix \ref{ap-dataset}.

\textbf{Evaluation metrics}. 
To evaluate the performance of heterogeneous peer effect (HPE) estimation, we use the \textit{Precision in the Estimation of Heterogeneous Effects} ({\small ${\epsilon_{PEHE}}$})~\cite{hill-jcgs11} metric defined as {\small ${\epsilon_{PEHE}} = \sqrt{\frac{1}{n}\textstyle \sum_i (\delta_i(\bm{\pi}_{-i}, \bm{\pi}'_{-i}) - \hat{\delta_i}(\bm{\pi}_{-i}, \bm{\pi}'_{-i}))^2},$} where {\small $\delta_i(\bm{\pi}_{-i}, \bm{\pi}'_{-i})$} is true HPE and {\small $\hat{\delta_i}(\bm{\pi}_{-i}, \bm{\pi}'_{-i})$} is the estimated HPE, where $\bm{\pi}'_{-i}$ denotes a counterfactual scenario where treatments of peers are flipped. {\small $\epsilon_{PEHE}$} (lower better) measures the deviation of estimated HPEs from true HPEs. For each experimental result, we report mean and standard deviation of {\small $\epsilon_{PEHE}$} for $5$ different simulations.

\textbf{Baselines}. We compare \ourmodel with state-of-the-art (SOTA) peer estimation methods. NetEst~\cite{jiang-cikm22} and TNet~\cite{chen-icml24} use the fraction of treated peers as peer exposure but the estimator is based on adversarial learning and doubly robust method, respectively, for robustness. DWR~\cite{zhao-arxiv22} learns attention weights based on attribute similarity and 1GNN-HSIC~\cite{ma-aistats21} use GNNs to summarize peer treatments as heterogeneous contexts while using homogeneous exposure. 
% GNN-TARNet-Motifs serve as references to check whether the exposure mapping function learned by our method is as good as or better than manually extracted causal network motifs.
\newchange{We also use the recently proposed GNN- and autoencoder-based automated exposure mapping approach (AEMNet)~\cite{mao-icassp25} and GNN- and transformer-based CauGramer~\cite{wu-iclr25} as baselines for estimating peer effects in our setup.} We also consider INE-TARNet~\cite{adhikari-mlj24} adapted for peer effect estimation as a baseline, although it was developed for direct effect estimation. We include GNN-TARNet-Motifs approach that considers manually extracted causal network motifs~\cite{yuan-www21} as peer exposure and TARNet as estimator~\cite{shalit-icml17} as a strong baseline. We discuss hyperparameter tuning and model selection in Appendix \ref{ap-hyperparams}.

\vspace{-1em}
\subsection{Results}

% \begin{figure*}
%     \centering
%     \subfigure[Mutual Connections]{
%     \label{fig:mut_frns}
%     \includegraphics[width=0.48\linewidth]{images/mutual_connections_kdd.jpg}
%     }
%     \subfigure[Clustering Coefficient]{
%     \label{fig:cc}
%     \includegraphics[width=0.48\linewidth]{images/clustering_coefficients_kdd.jpg}
%     }
%     \subfigure[Connected Components]{
%     \label{fig:div}
%     \includegraphics[width=0.48\linewidth]{images/connected_components_kdd.jpg}
%     }
%     \subfigure[Tie Strengths]{
%     \label{fig:tie}
%     \includegraphics[width=0.48\linewidth]{images/tie_strengths_kdd.jpg}
%     }
%     \caption{Peer effect estimation error for different underlying influence mechanisms. When true peer exposure depends on the number of mutual connections, our method comfortably outperforms all baselines showing its capability to count triangles in the ego network. When true peer exposure depends on clustering coefficient among treated peers, our method is better than or competitive to motif-count based baseline when the underlying peer exposure mechanism can be explained by causal motif counts. When true peer exposure depends on tie strengths between treated peers, our method is better than the motif-count based baseline for WS and BA networks but competitive to it for the SB network. }
%     \label{fig:syn}
% \end{figure*}

\begin{figure*}[!t]
    \centering
    \vspace{-2mm}\includegraphics[width=0.99\linewidth]{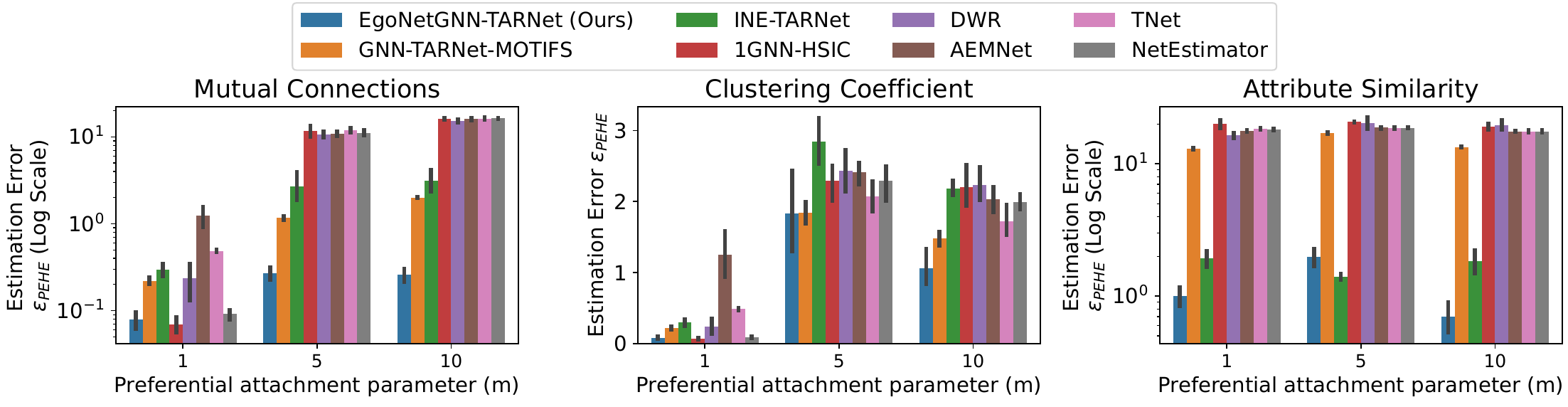}
    \vspace{-1em}
    \caption{Peer effect estimation error for \textbf{Barabasi Albert} network when true peer exposure depends on mutual connections, clustering coefficient, and attribute similarity. Our method shows robust performance across different underlying peer influence mechanisms and edge densities (low to high).}
    \label{fig:syn-ba-main}
    \vspace{-1em}
\end{figure*}

\begin{table}[!t]
    \caption{Mean and standard deviation of peer effect estimation error ($\epsilon_{PEHE}$) for different methods in BlogCatalog (BC) dataset for four settings when true peer exposure mechanisms depend on clustering coefficients, connected components, mutual connections, and attribute similarity.}
    \label{tab:exp-semi}
    \centering
\resizebox{\textwidth}{!}{
\begin{tabular}{p{5.1em}p{3em}p{3em}p{3em}p{3em}p{3em}p{3.1em}p{3.1em}p{3em}p{3em}p{3em}}
\toprule
{Mechanisms} & Ours-TARNet & \newchange{Ours-CFR} & GNN-Motifs &        INE-TARNet &          1GNN-HSIC &                DWR &             AEMNet &                TNet &       NetEst & \newchange{CauGramer}\\
\midrule
%  \multicolumn{9}{c}{$\epsilon_{PEHE}$ in the BlogCatalog Dataset}\\
% \midrule
Clus. Coef. &  {\small $\underline{1.59}_{\pm 0.4}$} &  {\small $\mathbf{1.51}_{\pm 0.3}$} &  {\small ${2.04}_{\pm 0.7}$} &  {\small $2.24_{\pm 0.6}$} &   {\small $5.83_{\pm 3.7}$} &   {\small $5.95_{\pm 1.8}$} &   {\small $4.60_{\pm 2.2}$} &    {\small $8.88_{\pm 9.5}$} &   {\small $3.47_{\pm 0.3}$} &   {\small $6.31_{\pm 2.1}$} \\
Con. Comp.    &  {\small $\underline{2.98}_{\pm 0.8}$} & {\small $\mathbf{2.77}_{\pm 0.9}$} &  {\small $4.41_{\pm 1.1}$} &  {\small ${3.93}_{\pm 0.9}$} &   {\small $6.83_{\pm 1.3}$} &   {\small $6.84_{\pm 1.4}$} &   {\small $8.05_{\pm 4.5}$} &  {\small $11.60_{\pm 11.5}$} &   {\small $7.38_{\pm 0.8}$} &   {\small $7.09_{\pm 1.0}$}\\
Mut. Con.      &  {\small $\underline{2.90}_{\pm 1.1}$} & {\small $\mathbf{2.57}_{\pm 0.8}$} &  {\small $3.83_{\pm 0.7}$} &  {\small ${3.37}_{\pm 0.8}$} &   {\small $9.29_{\pm 4.7}$} &   {\small $8.21_{\pm 3.3}$} &  {\small $11.71_{\pm 4.8}$} &  {\small $13.15_{\pm 12.4}$} &   {\small $6.62_{\pm 2.2}$} &   {\small $8.24_{\pm 1.4}$}\\
Attr. Sim.    &  {\small $\underline{5.65}_{\pm 0.7}$}  &  {\small $\mathbf{4.86}_{\pm 1.9}$} &  {\small $6.09_{\pm 0.2}$} &  {\small ${6.01}_{\pm 2.0}$} &  {\small $17.48_{\pm 9.1}$} &  {\small $15.69_{\pm 6.7}$} &  {\small $21.80_{\pm 9.3}$} &   {\small $15.70_{\pm 7.7}$} &  {\small $14.07_{\pm 5.8}$} &  {\small $7.34_{\pm 3.8}$} \\
% \midrule
%  \multicolumn{9}{c}{$\epsilon_{PEHE}$ in the Flickr Dataset}\\
% \midrule
% \multirow{4}{*}{FL} & Clus. Coef. &   {\small $\mathbf{4.93}_{\pm 1.6}$} &    {\small $5.34_{\pm 1.5}$} &   {\small $\underline{5.26}_{\pm 1.6}$} &   {\small $9.56_{\pm 4.9}$} &   {\small $9.51_{\pm 2.2}$} &    {\small $8.05_{\pm 5.5}$} &    {\small $9.75_{\pm 4.6}$} &   {\small $7.57_{\pm 1.3}$} \\
% & Con. Comp.    &   {\small $\mathbf{1.83}_{\pm 0.6}$} &    {\small $2.80_{\pm 1.2}$} &   {\small $\underline{1.85}_{\pm 0.7}$} &   {\small $3.36_{\pm 0.8}$} &   {\small $2.75_{\pm 0.6}$} &    {\small $4.69_{\pm 1.7}$} &    {\small $2.94_{\pm 0.9}$} &   {\small $2.67_{\pm 0.5}$} \\
% & Mut. Con.      &   {\small $\underline{2.38}_{\pm 1.3}$} &    {\small $2.55_{\pm 0.5}$} &   {\small $\mathbf{2.36}_{\pm 0.6}$} &   {\small $4.03_{\pm 1.6}$} &   {\small $3.57_{\pm 1.7}$} &  {\small $10.95_{\pm 12.3}$} &  {\small $10.96_{\pm 17.2}$} &   {\small $4.24_{\pm 1.8}$} \\
% & Attr. Sim.    &  {\small $\mathbf{11.32}_{\pm 6.6}$} &  {\small $13.15_{\pm 10.8}$} &  {\small $\underline{12.17}_{\pm 8.8}$} &  {\small $16.94_{\pm 8.1}$} &  {\small $18.03_{\pm 9.7}$} &  {\small $17.43_{\pm 10.0}$} &  {\small $23.09_{\pm 20.3}$} &  {\small $16.87_{\pm 7.8}$} \\

 \bottomrule
\end{tabular}
}
\vspace{-1em}
\end{table}

\begin{table}[!t]
    \caption{Mean and standard deviation HPE estimation error ($\epsilon_{PEHE}$ metric) for three variants of our method (original, without mask, and without feature encoder and mask) in the BlogCatalog (BC), Barabasi Albert (BA), and Watts Strogatz (WS) datasets for three true peer exposure mechanisms.}
    \label{tab:ablation}
    \centering
\resizebox{\textwidth}{!}{
\begin{tabular}{llllllllll}
\toprule
Mechanism & \multicolumn{3}{l}{Mutual Connections} & \multicolumn{3}{l}{Clustering Coefficient} & \multicolumn{3}{l}{Attribute Similarity} \\
Network &        BC & BA & WS &            BC & BA & WS &          BC & BA & WS \\
Model Variants                  &                    &                 &                &                        &                 &                &                      &                 &                \\
\midrule
Ours       &  {\small $\mathbf{2.90}_{\pm 1.1}$} &  {\small $\mathbf{0.20}_{\pm 0.1}$} &  {\small $\mathbf{0.30}_{\pm 0.1}$} &  {\small $1.59_{\pm 0.4}$} &  {\small $0.99_{\pm 0.9}$} &  {\small $\mathbf{1.18}_{\pm 0.8}$} &  {\small $\mathbf{5.65}_{\pm 0.7}$} &   {\small $\mathbf{1.23}_{\pm 0.7}$} &   {\small $\mathbf{1.09}_{\pm 1.1}$} \\
Ours (-mask)   &  {\small $2.91_{\pm 1.5}$} &  {\small $0.21_{\pm 0.1}$} &  {\small $0.35_{\pm 0.2}$} &  {\small $1.98_{\pm 1.1}$} &  {\small $1.01_{\pm 0.8}$} &  {\small $1.20_{\pm 0.4}$} &  {\small $6.17_{\pm 0.6}$} &   {\small $2.37_{\pm 2.2}$} &   {\small $1.29_{\pm 1.8}$} \\
Ours (-feat\&mask) &  {\small $\mathbf{2.90}_{\pm 0.7}$} &  {\small $0.27_{\pm 0.2}$} &  {\small $0.31_{\pm 0.1}$} &  {\small $\mathbf{1.55}_{\pm 0.5}$} &  {\small $\mathbf{0.97}_{\pm 0.7}$} &  {\small $1.91_{\pm 1.3}$} &  {\small $5.99_{\pm 0.8}$} &  {\small $13.73_{\pm 2.8}$} &  {\small $13.85_{\pm 4.0}$} \\

 \bottomrule
\end{tabular}
}
\vspace{-1em}
\end{table}

\begin{table}[!t]
    \caption{\newchange{Evaluation of exposure representation, in terms of absolute correlation, in BlogCatalog data with no effect modification. The results for the learned peer exposure representation by our method is better (higher is better). We use the fraction of treated friends $z_i$ as baseline and the dimension of $\hat{\bm{\rho}},\hat{\bm{\rho}}' \in [0,1]^{d=2}$ with highest correlation is shown.}}
    \label{tab:direct_eval}
    \centering
    \resizebox{\textwidth}{!}{
\begin{tabular}{llllllllll}
\toprule
Corr.&Clus. Coef.&Con. Comp.&Mut. Con.&Attr. Sim.&Corr.&Clus. Coef.&Con. Comp.&Mut. Con.&Attr. Sim.\\
\midrule
$r(\hat{\rho},\rho)$& $\mathbf{0.81}_{\pm 0.1}$ &  $\mathbf{0.34}_{\pm 0.3}$  & $\mathbf{0.73}_{\pm 0.2}$ & $\mathbf{0.29}_{\pm 0.2}$ & $r(\hat{\rho}', \rho')$   & $\mathbf{0.85}_{\pm 0.02}$&$\mathbf{0.30}_{\pm 0.2}$ & $\mathbf{0.74}_{\pm 0.1}$ & $0.50_{\pm 0.1}$ \\
$r(z_i, \rho)$   & $0.17_{\pm 0.1}$&$0.12_{\pm 0.1}$ &$0.09_{\pm 0.03}$ & $0.28_{\pm 0.2}$ & $r(z_i', \rho')$   &$0.41_{\pm 0.2}$&$0.14_{\pm 0.1}$ & $0.09_{\pm 0.1}$ & $\mathbf{0.61}_{\pm 0.1}$ \\
\bottomrule
    \end{tabular}
}
\vspace{-1em}
\end{table}
% \begin{figure*}[!t]
%     \centering
%     \includegraphics[width=\linewidth]{images/mutual_connections_kdd.jpg}
%     \caption{{Peer effect estimation error when true peer exposure depends on number of mutual connections. Our method significantly outperforms all baselines showing its capability to count triangles in the ego network.}}
%     \label{fig:mut_frns}
% \end{figure*}
% \begin{figure*}[!t]
%     \centering
%     \includegraphics[width=\linewidth]{images/connected_components_kdd.jpg}
%     \caption{{Peer effect estimation error when true peer exposure depends on connected components among treated peers. Our method performs well compared to all baselines when underlying peer exposure mechanism cannot be explained totally with motifs structures only.}}
%     \label{fig:div}
% \end{figure*}
Next, we present results for experimental setups designed to answer \newchange{four} research questions (RQs).
\textbf{RQ1. How well do methods for peer effect estimation perform when peer exposure mechanisms depend on local neighborhood conditions?}
% We generate synthetic networks, BA and WS, with low, medium, and high edge density and SBM network with different block sizes.
In this setup, we evaluate the performance of peer effect estimators when the underlying peer exposure mechanism is unknown. We generate treatments and outcomes such that there is confounding due to a subset of node attributes and mean peer attributes. For the outcome generation, we consider five mechanisms for true peer exposure conditions where peer exposure is given by 1) the clustering coefficient between the treated peers, 2) the number of connected components among treated peers, and weighted fraction of treated peers with weights as 3) the square root of number of mutual connections, 4) attribute similarity, and 5) tie strength. Here, the unit's treatment acts as an effect modifier, where the peer exposure is doubled if the unit is treated.
%The coefficients scaling peer effects $\delta_{exp}$ and $\delta_{em}$ are set to $20$ for the first, second and fourth mechanisms and $1$ for the third mechanism because true peer exposure in the former case are bounded from 0 to 1 while the later one is unbounded. }
Figure \ref{fig:syn-ba-main} shows peer effect estimation error (y-axis) when true peer exposure mechanisms depend on mutual connections, clustering coefficient, and attribute similarity in Barabasi Albert networks with three network generation parameters (x-axis) resulting different edge densities (low to high). The preferential attachment parameter $m=1$ produces a sparse star-topology network, lacking cycles or triangular structures. In this setting, all methods perform relatively well when peer exposure mechanisms depend on local structure because MPGNNs are expressive enough to capture star-shaped motifs. However, with increased edge density and more complex network topology, unlike our method, the baselines are not sufficiently expressive to capture underlying mechanisms and suffer significantly. The GNN-TARNet-Motifs (GTM) approach is expressive in capturing clustering coefficients, and both GTM and INE-TARNet approximate mutual connections. This is reflected in the performance, where GTM is competitive for the clustering coefficient peer exposure mechanism. \newchange{\ourmodel-TARNet outperforms the baselines except for INE-TARNet, which is competitive in a setting with the peer exposure mechanism dependent on attribute similarity. Figure \ref{fig:syn-ba-main} and other results in Appendix \ref{ap-result-synthetic} show that for unknown peer exposure mechanisms, our method is as expressive as or superior to the strongest baseline with significantly better performance for denser networks.}
% These results generalize to other peer exposure mechanisms, including tie strength, and synthetic networks, as shown in Appendix \ref{ap-result-synthetic}.

{\textbf{RQ2. How reliable are the models for heterogeneous peer effect estimation in more realistic scenario?}}
RQ2 investigates the performance of the models using more realistic semi-synthetic networks and node attributes. In addition to confounding and heterogeneous peer influence, there is more complex peer effect modification depending on whether the unit is treated and the values of the unit's attributes. Table \ref{tab:exp-semi} shows mean and standard deviation of peer effect estimation error ($\epsilon_{PEHE}$) for different methods in \newchange{BlogCatalog (BC) dataset for four settings when true peer exposure mechanisms depend on clustering coefficients, connected components, mutual connections, and attribute similarity. The results show the robustness of \ourmodel in a more realistic setting, where \ourmodel-CFR outperforms all baselines. \ourmodel-TARNet is competitive with the best model and it also outperforms the baselines. In this setup, peer effects are heterogeneous due to the interaction of peer exposure conditions and effect modifiers, and our method is able to approximate it better than the baselines. {Appendix} \ref{ap-result-semi} presents additional experiments for this setup including results for the Flickr dataset (Table \ref{tab:exp-semi-flickr}) which is more challenging for the baselines.}
% For the fifth mechanism, we consider peer exposure given by weighted fraction of treated peers with weights as peer degree since the semi-synthetic data lack edge attributes.

\textbf{RQ3. How do the components of \ourmodel contribute to its robustness in estimating peer effects?} We conduct ablation studies to assess the contributions of masked weights and the feature encoder MLP. Table \ref{tab:ablation} displays the performance of three variants of \ourmodel-TARNet (original, without the masked weights, and without the feature encoder and masked weights) across BlogCatalog (BC), Barabási Albert (BA), and Watts-Strogatz (WS) datasets. The results show that excluding masked weights can bias peer effect estimates due to the model's sensitivity to irrelevant contexts. Removing the feature encoder MLP limits EgoNetGNN's ability to capture mechanisms based on attribute similarity. However, certain peer exposure mechanisms relying on local structures perform better when irrelevant features are ignored. Overall, these findings demonstrate that the feature encoder MLP enhances expressiveness, while masked weights promote invariance to irrelevant contexts.
Additionally, we analyze the \ourmodel's sensitivity to the choices of peer exposure embedding dimension, \newchange{coverage} loss coefficient, and noisy networks in Appendix \ref{ap-result-ablation}.

\newchange{\textbf{RQ4. How well are underlying mechanisms captured by learned exposure mapping function?} In Table 3, we directly compare the (absolute) Pearson correlation coefficient $r$ (higher is better) between the learned peer exposure representation, $\hat{\bm{\rho}}$  and $\hat{\bm{\rho}'}$, and the actual peer exposure under four different mechanisms.  Compared to commonly used fraction of treated friends baseline, learned peer exposures are informative of true peer exposures for mechanisms involving local structure. 
% We note that the peer exposures could be informative even if they are not directly (linearly) correlated with true values (e.g., if they cluster similarly).
}

\vspace{-1em}
\section{Discussion, Limitations  \& Future Work}
% \vspace{-1em}
Our work motivates the problem of learning exposure mapping function for peer effect estimation and proposes \ourmodel for addressing unknown peer influence mechanisms involving local neighborhood conditions. Our theoretical analysis and experimental results demonstrate increased expressiveness of \ourmodel to capture complex local neighborhood exposure conditions. We have designed \ourmodel to promote invariance to irrelevant contexts, and output a low-dimensional peer exposure embedding with bounded and balanced representation to partially mitigate issue of potential violation of the positivity assumption with continuous treatment or exposure. The empirical results have shown the effectiveness of \ourmodel in many peer effect estimation settings.
 
 \textbf{Limitations \& Future Work. }Ensuring theoretical bounds for variance with complex GNNs for heterogeneous causal effect estimation is still a developing research area~\cite{khatami-arxiv24} and important future direction, but it is not within the scope of our current work. This work can be extended to incorporate other network effects like direct effects and total effects.
 % and generalized to incorporate contagion effects where peer treatments affect their outcomes which in turn affect a unit's outcome in another timestep. 
 The increased expressiveness and robust peer effect estimates of our model come with the trade-off of a slightly longer runtime to process ego networks. Future work could consider relaxing the assumption of interference from immediate peers while addressing the scalability. Our work relies upon a reliable attributed network as input, but future research should consider capturing expressive representations in noisy networks. {Appendix \ref{ap-discussion} discusses societal impacts, scalability, and plausibility of assumptions.}
 
 {\section*{Reproducibility Statement}} 

 To support reproducibility, we release the complete codebase and experimental procedures. For all the experiments, we have repeated them at least \textit{five} times. We provide the details of the data generation process (Sec. \ref{sec:exp_setup} and Appendix \ref{ap-dataset}). Appendix starts with an anonymous repository link containing the full source code. We provide the details of the configurations and setups for replicating our results in Appendix \ref{ap-hyperparams}.

\bibliographystyle{ACM-Reference-Format}
\bibliography{references}

%%% -*-BibTeX-*-
%%% Do NOT edit. File created by BibTeX with style
%%% ACM-Reference-Format-Journals [18-Jan-2012].

\begin{thebibliography}{45}

%%% ====================================================================
%%% NOTE TO THE USER: you can override these defaults by providing
%%% customized versions of any of these macros before the \bibliography
%%% command.  Each of them MUST provide its own final punctuation,
%%% except for \shownote{}, \showDOI{}, and \showURL{}.  The latter two
%%% do not use final punctuation, in order to avoid confusing it with
%%% the Web address.
%%%
%%% To suppress output of a particular field, define its macro to expand
%%% to an empty string, or better, \unskip, like this:
%%%
%%% \newcommand{\showDOI}[1]{\unskip}   % LaTeX syntax
%%%
%%% \def \showDOI #1{\unskip}           % plain TeX syntax
%%%
%%% ====================================================================

\ifx \showCODEN    \undefined \def \showCODEN     #1{\unskip}     \fi
\ifx \showDOI      \undefined \def \showDOI       #1{#1}\fi
\ifx \showISBNx    \undefined \def \showISBNx     #1{\unskip}     \fi
\ifx \showISBNxiii \undefined \def \showISBNxiii  #1{\unskip}     \fi
\ifx \showISSN     \undefined \def \showISSN      #1{\unskip}     \fi
\ifx \showLCCN     \undefined \def \showLCCN      #1{\unskip}     \fi
\ifx \shownote     \undefined \def \shownote      #1{#1}          \fi
\ifx \showarticletitle \undefined \def \showarticletitle #1{#1}   \fi
\ifx \showURL      \undefined \def \showURL       {\relax}        \fi
% The following commands are used for tagged output and should be
% invisible to TeX
\providecommand\bibfield[2]{#2}
\providecommand\bibinfo[2]{#2}
\providecommand\natexlab[1]{#1}
\providecommand\showeprint[2][]{arXiv:#2}

\bibitem[Adhikari and Zheleva(2025)]%
        {adhikari-mlj24}
\bibfield{author}{\bibinfo{person}{Shishir Adhikari} {and} \bibinfo{person}{Elena Zheleva}.} \bibinfo{year}{2025}\natexlab{}.
\newblock \showarticletitle{Inferring Individual Direct Causal Effects Under Heterogeneous Peer Influence}.
\newblock \bibinfo{journal}{\emph{Machine Learning Journal}} (\bibinfo{year}{2025}).
\newblock


\bibitem[Albert and Barab{\'a}si(2002)]%
        {albert-rmp02}
\bibfield{author}{\bibinfo{person}{R{\'e}ka Albert} {and} \bibinfo{person}{Albert-L{\'a}szl{\'o} Barab{\'a}si}.} \bibinfo{year}{2002}\natexlab{}.
\newblock \showarticletitle{Statistical mechanics of complex networks}.
\newblock \bibinfo{journal}{\emph{Reviews of modern physics}} \bibinfo{volume}{74}, \bibinfo{number}{1} (\bibinfo{year}{2002}), \bibinfo{pages}{47}.
\newblock


\bibitem[Angrist et~al\mbox{.}(1996)]%
        {angrist-jasa96}
\bibfield{author}{\bibinfo{person}{Joshua~D Angrist}, \bibinfo{person}{Guido~W Imbens}, {and} \bibinfo{person}{Donald~B Rubin}.} \bibinfo{year}{1996}\natexlab{}.
\newblock \showarticletitle{Identification of causal effects using instrumental variables}.
\newblock \bibinfo{journal}{\emph{Journal of the American statistical Association}} \bibinfo{volume}{91}, \bibinfo{number}{434} (\bibinfo{year}{1996}), \bibinfo{pages}{444--455}.
\newblock


\bibitem[Arbour et~al\mbox{.}(2016)]%
        {arbour-kdd16}
\bibfield{author}{\bibinfo{person}{David Arbour}, \bibinfo{person}{Dan Garant}, {and} \bibinfo{person}{David Jensen}.} \bibinfo{year}{2016}\natexlab{}.
\newblock \showarticletitle{Inferring network effects from observational data}. In \bibinfo{booktitle}{\emph{Proceedings of the 22nd ACM SIGKDD International Conference on Knowledge Discovery and Data Mining}}. \bibinfo{pages}{715--724}.
\newblock


\bibitem[Arjovsky et~al\mbox{.}(2017)]%
        {arjovsky-icml17}
\bibfield{author}{\bibinfo{person}{Martin Arjovsky}, \bibinfo{person}{Soumith Chintala}, {and} \bibinfo{person}{L{\'e}on Bottou}.} \bibinfo{year}{2017}\natexlab{}.
\newblock \showarticletitle{Wasserstein generative adversarial networks}. In \bibinfo{booktitle}{\emph{International conference on machine learning}}. PMLR, \bibinfo{pages}{214--223}.
\newblock


\bibitem[Aronow and Samii(2017)]%
        {aronow-aas17}
\bibfield{author}{\bibinfo{person}{Peter~M Aronow} {and} \bibinfo{person}{Cyrus Samii}.} \bibinfo{year}{2017}\natexlab{}.
\newblock \showarticletitle{Estimating average causal effects under general interference, with application to a social network experiment}.
\newblock \bibinfo{journal}{\emph{The Annals of Applied Statistics}} \bibinfo{volume}{11}, \bibinfo{number}{4} (\bibinfo{year}{2017}), \bibinfo{pages}{1912--1947}.
\newblock


\bibitem[Auerbach et~al\mbox{.}(2024)]%
        {auerbach-arxiv24}
\bibfield{author}{\bibinfo{person}{Eric Auerbach}, \bibinfo{person}{Jonathan Auerbach}, {and} \bibinfo{person}{Max Tabord-Meehan}.} \bibinfo{year}{2024}\natexlab{}.
\newblock \showarticletitle{Exposure effects are not automatically useful for policymaking}.
\newblock \bibinfo{journal}{\emph{arXiv preprint arXiv:2401.06264}} (\bibinfo{year}{2024}).
\newblock


\bibitem[Bargagli-Stoffi et~al\mbox{.}(2025)]%
        {bargagli-aas25}
\bibfield{author}{\bibinfo{person}{Falco~J Bargagli-Stoffi}, \bibinfo{person}{Costanza Tort{\'u}}, {and} \bibinfo{person}{Laura Forastiere}.} \bibinfo{year}{2025}\natexlab{}.
\newblock \showarticletitle{Heterogeneous treatment and spillover effects under clustered network interference}.
\newblock \bibinfo{journal}{\emph{The annals of applied statistics}} \bibinfo{volume}{19}, \bibinfo{number}{1} (\bibinfo{year}{2025}), \bibinfo{pages}{28}.
\newblock


\bibitem[Barkley et~al\mbox{.}(2020)]%
        {barkley-aas20}
\bibfield{author}{\bibinfo{person}{Brian~G Barkley}, \bibinfo{person}{Michael~G Hudgens}, \bibinfo{person}{John~D Clemens}, \bibinfo{person}{Mohammad Ali}, {and} \bibinfo{person}{Michael~E Emch}.} \bibinfo{year}{2020}\natexlab{}.
\newblock \showarticletitle{Causal inference from observational studies with clustered interference, with application to a cholera vaccine study}.
\newblock \bibinfo{journal}{\emph{Annals of Applied Statistics}} \bibinfo{volume}{14}, \bibinfo{number}{3} (\bibinfo{year}{2020}), \bibinfo{pages}{1432--1448}.
\newblock


\bibitem[Blakely et~al\mbox{.}({[n.\,d.]})]%
        {blakely-github2021}
\bibfield{author}{\bibinfo{person}{Derrick Blakely}, \bibinfo{person}{Jack Lanchantin}, {and} \bibinfo{person}{Yanjun Qi}.} \bibinfo{year}{[n.\,d.]}\natexlab{}.
\newblock \bibinfo{title}{Time and space complexity of graph convolutional networks}.
\newblock \bibinfo{howpublished}{\url{https://qdata.github.io/deep2Read/talks-mb2019/Derrick_201906_GCN_complexityAnalysis-writeup.pdf}}.
\newblock


\bibitem[Blei et~al\mbox{.}(2003)]%
        {blei-jmlr03}
\bibfield{author}{\bibinfo{person}{David~M Blei}, \bibinfo{person}{Andrew~Y Ng}, {and} \bibinfo{person}{Michael~I Jordan}.} \bibinfo{year}{2003}\natexlab{}.
\newblock \showarticletitle{Latent dirichlet allocation}.
\newblock \bibinfo{journal}{\emph{Journal of machine Learning research}} \bibinfo{volume}{3}, \bibinfo{number}{Jan} (\bibinfo{year}{2003}), \bibinfo{pages}{993--1022}.
\newblock


\bibitem[Cai et~al\mbox{.}(2023)]%
        {cai-cikm23}
\bibfield{author}{\bibinfo{person}{Ruichu Cai}, \bibinfo{person}{Zeqin Yang}, \bibinfo{person}{Weilin Chen}, \bibinfo{person}{Yuguang Yan}, {and} \bibinfo{person}{Zhifeng Hao}.} \bibinfo{year}{2023}\natexlab{}.
\newblock \showarticletitle{Generalization bound for estimating causal effects from observational network data}. In \bibinfo{booktitle}{\emph{CIKM}}. \bibinfo{pages}{163--172}.
\newblock


\bibitem[Chen et~al\mbox{.}(2024)]%
        {chen-icml24}
\bibfield{author}{\bibinfo{person}{Weilin Chen}, \bibinfo{person}{Ruichu Cai}, \bibinfo{person}{Zeqin Yang}, \bibinfo{person}{Jie Qiao}, \bibinfo{person}{Yuguang Yan}, \bibinfo{person}{Zijian Li}, {and} \bibinfo{person}{Zhifeng Hao}.} \bibinfo{year}{2024}\natexlab{}.
\newblock \showarticletitle{Doubly Robust Causal Effect Estimation under Networked Interference via Targeted Learning}. In \bibinfo{booktitle}{\emph{Forty-first International Conference on Machine Learning}}.
\newblock


\bibitem[Chen et~al\mbox{.}(2020)]%
        {chen-neurips20}
\bibfield{author}{\bibinfo{person}{Zhengdao Chen}, \bibinfo{person}{Lei Chen}, \bibinfo{person}{Soledad Villar}, {and} \bibinfo{person}{Joan Bruna}.} \bibinfo{year}{2020}\natexlab{}.
\newblock \showarticletitle{Can graph neural networks count substructures?}
\newblock \bibinfo{journal}{\emph{Advances in neural information processing systems}}  \bibinfo{volume}{33} (\bibinfo{year}{2020}), \bibinfo{pages}{10383--10395}.
\newblock


\bibitem[Cuturi and Doucet(2014)]%
        {cuturi-icml14}
\bibfield{author}{\bibinfo{person}{Marco Cuturi} {and} \bibinfo{person}{Arnaud Doucet}.} \bibinfo{year}{2014}\natexlab{}.
\newblock \showarticletitle{Fast computation of Wasserstein barycenters}. In \bibinfo{booktitle}{\emph{International conference on machine learning}}. PMLR, \bibinfo{pages}{685--693}.
\newblock


\bibitem[Forastiere et~al\mbox{.}(2021)]%
        {forastiere-asa21}
\bibfield{author}{\bibinfo{person}{Laura Forastiere}, \bibinfo{person}{Edoardo~M Airoldi}, {and} \bibinfo{person}{Fabrizia Mealli}.} \bibinfo{year}{2021}\natexlab{}.
\newblock \showarticletitle{Identification and estimation of treatment and interference effects in observational studies on networks}.
\newblock \bibinfo{journal}{\emph{J. Amer. Statist. Assoc.}} \bibinfo{volume}{116}, \bibinfo{number}{534} (\bibinfo{year}{2021}), \bibinfo{pages}{901--918}.
\newblock


\bibitem[Hill(2011)]%
        {hill-jcgs11}
\bibfield{author}{\bibinfo{person}{Jennifer~L Hill}.} \bibinfo{year}{2011}\natexlab{}.
\newblock \showarticletitle{Bayesian nonparametric modeling for causal inference}.
\newblock \bibinfo{journal}{\emph{Journal of Computational and Graphical Statistics}} \bibinfo{volume}{20}, \bibinfo{number}{1} (\bibinfo{year}{2011}), \bibinfo{pages}{217--240}.
\newblock


\bibitem[Hudgens and Halloran(2008)]%
        {hudgens-jasa08}
\bibfield{author}{\bibinfo{person}{Michael~G Hudgens} {and} \bibinfo{person}{M~Elizabeth Halloran}.} \bibinfo{year}{2008}\natexlab{}.
\newblock \showarticletitle{Toward causal inference with interference}.
\newblock \bibinfo{journal}{\emph{J. Amer. Statist. Assoc.}} \bibinfo{volume}{103}, \bibinfo{number}{482} (\bibinfo{year}{2008}), \bibinfo{pages}{832--842}.
\newblock


\bibitem[Im et~al\mbox{.}(2021)]%
        {im-arxiv21}
\bibfield{author}{\bibinfo{person}{Daniel~Jiwoong Im}, \bibinfo{person}{Kyunghyun Cho}, {and} \bibinfo{person}{Narges Razavian}.} \bibinfo{year}{2021}\natexlab{}.
\newblock \showarticletitle{Causal effect variational autoencoder with uniform treatment}.
\newblock \bibinfo{journal}{\emph{arXiv preprint arXiv:2111.08656}} (\bibinfo{year}{2021}).
\newblock


\bibitem[Jiang and Sun(2022)]%
        {jiang-cikm22}
\bibfield{author}{\bibinfo{person}{Song Jiang} {and} \bibinfo{person}{Yizhou Sun}.} \bibinfo{year}{2022}\natexlab{}.
\newblock \showarticletitle{Estimating Causal Effects on Networked Observational Data via Representation Learning}. In \bibinfo{booktitle}{\emph{Proceedings of the 31st ACM International Conference on Information \& Knowledge Management}}. \bibinfo{pages}{852--861}.
\newblock


\bibitem[Khatami et~al\mbox{.}(2024)]%
        {khatami-arxiv24}
\bibfield{author}{\bibinfo{person}{Seyedeh~Baharan Khatami}, \bibinfo{person}{Harsh Parikh}, \bibinfo{person}{Haowei Chen}, \bibinfo{person}{Sudeepa Roy}, {and} \bibinfo{person}{Babak Salimi}.} \bibinfo{year}{2024}\natexlab{}.
\newblock \bibinfo{title}{Graph Machine Learning based Doubly Robust Estimator for Network Causal Effects}.
\newblock
\newblock
\showeprint[arxiv]{2403.11332}~[cs.LG]
\urldef\tempurl%
\url{https://arxiv.org/abs/2403.11332}
\showURL{%
\tempurl}


\bibitem[Kipf and Welling(2016)]%
        {kipf-iclr16}
\bibfield{author}{\bibinfo{person}{Thomas~N Kipf} {and} \bibinfo{person}{Max Welling}.} \bibinfo{year}{2016}\natexlab{}.
\newblock \showarticletitle{Semi-Supervised Classification with Graph Convolutional Networks}. In \bibinfo{booktitle}{\emph{International Conference on Learning Representations}}.
\newblock


\bibitem[Leung and Loupos(2022)]%
        {leung-arxiv22}
\bibfield{author}{\bibinfo{person}{Michael~P Leung} {and} \bibinfo{person}{Pantelis Loupos}.} \bibinfo{year}{2022}\natexlab{}.
\newblock \showarticletitle{Unconfoundedness with network interference}.
\newblock \bibinfo{journal}{\emph{arXiv preprint arXiv:2211.07823}}  \bibinfo{volume}{6} (\bibinfo{year}{2022}).
\newblock


\bibitem[Lin et~al\mbox{.}(2023)]%
        {lin-ecmlkdd23}
\bibfield{author}{\bibinfo{person}{Xiaofeng Lin}, \bibinfo{person}{Guoxi Zhang}, \bibinfo{person}{Xiaotian Lu}, \bibinfo{person}{Han Bao}, \bibinfo{person}{Koh Takeuchi}, {and} \bibinfo{person}{Hisashi Kashima}.} \bibinfo{year}{2023}\natexlab{}.
\newblock \showarticletitle{Estimating Treatment Effects Under Heterogeneous Interference}. In \bibinfo{booktitle}{\emph{Joint European Conference on ML and KDD}}. Springer, \bibinfo{pages}{576--592}.
\newblock


\bibitem[Lin et~al\mbox{.}(2024)]%
        {lin-packdd24}
\bibfield{author}{\bibinfo{person}{Xiaofeng Lin}, \bibinfo{person}{Guoxi Zhang}, \bibinfo{person}{Xiaotian Lu}, {and} \bibinfo{person}{Hisashi Kashima}.} \bibinfo{year}{2024}\natexlab{}.
\newblock \showarticletitle{Treatment Effect Estimation Under Unknown Interference}. In \bibinfo{booktitle}{\emph{Pacific-Asia Conference on Knowledge Discovery and Data Mining}}. Springer, \bibinfo{pages}{28--42}.
\newblock


\bibitem[Ma et~al\mbox{.}(2022)]%
        {ma-kdd22}
\bibfield{author}{\bibinfo{person}{Jing Ma}, \bibinfo{person}{Mengting Wan}, \bibinfo{person}{Longqi Yang}, \bibinfo{person}{Jundong Li}, \bibinfo{person}{Brent Hecht}, {and} \bibinfo{person}{Jaime Teevan}.} \bibinfo{year}{2022}\natexlab{}.
\newblock \showarticletitle{Learning causal effects on hypergraphs}. In \bibinfo{booktitle}{\emph{Proceedings of the 28th ACM SIGKDD Conference on Knowledge Discovery and Data Mining}}. \bibinfo{pages}{1202--1212}.
\newblock


\bibitem[Ma and Tresp(2021)]%
        {ma-aistats21}
\bibfield{author}{\bibinfo{person}{Yunpu Ma} {and} \bibinfo{person}{Volker Tresp}.} \bibinfo{year}{2021}\natexlab{}.
\newblock \showarticletitle{Causal inference under networked interference and intervention policy enhancement}. In \bibinfo{booktitle}{\emph{International Conference on Artificial Intelligence and Statistics}}. PMLR, \bibinfo{pages}{3700--3708}.
\newblock


\bibitem[Mao et~al\mbox{.}(2025)]%
        {mao-icassp25}
\bibfield{author}{\bibinfo{person}{Yunxin Mao}, \bibinfo{person}{Haotian Wang}, \bibinfo{person}{Yishuai Cai}, \bibinfo{person}{Minglong Li}, \bibinfo{person}{Ji Wang}, {and} \bibinfo{person}{Wenjing Yang}.} \bibinfo{year}{2025}\natexlab{}.
\newblock \showarticletitle{Automated Exposure Mapping for Networked Interference}. In \bibinfo{booktitle}{\emph{ICASSP 2025-2025 IEEE International Conference on Acoustics, Speech and Signal Processing (ICASSP)}}. IEEE, \bibinfo{pages}{1--5}.
\newblock


\bibitem[Miao et~al\mbox{.}(2024)]%
        {miao-strf24}
\bibfield{author}{\bibinfo{person}{Wang Miao}, \bibinfo{person}{Xu Shi}, \bibinfo{person}{Yilin Li}, {and} \bibinfo{person}{Eric~J Tchetgen~Tchetgen}.} \bibinfo{year}{2024}\natexlab{}.
\newblock \showarticletitle{A confounding bridge approach for double negative control inference on causal effects}.
\newblock \bibinfo{journal}{\emph{Statistical Theory and Related Fields}} \bibinfo{volume}{8}, \bibinfo{number}{4} (\bibinfo{year}{2024}), \bibinfo{pages}{262--273}.
\newblock


\bibitem[Nabi et~al\mbox{.}(2022)]%
        {nabi-frontiers22}
\bibfield{author}{\bibinfo{person}{Razieh Nabi}, \bibinfo{person}{Joel Pfeiffer}, \bibinfo{person}{Denis Charles}, {and} \bibinfo{person}{Emre K{\i}c{\i}man}.} \bibinfo{year}{2022}\natexlab{}.
\newblock \showarticletitle{Causal inference in the presence of interference in sponsored search advertising}.
\newblock \bibinfo{journal}{\emph{Frontiers in big Data}}  \bibinfo{volume}{5} (\bibinfo{year}{2022}).
\newblock


\bibitem[Ogburn et~al\mbox{.}(2022)]%
        {ogburn-jasa22}
\bibfield{author}{\bibinfo{person}{Elizabeth~L Ogburn}, \bibinfo{person}{Oleg Sofrygin}, \bibinfo{person}{Ivan Diaz}, {and} \bibinfo{person}{Mark~J Van~der Laan}.} \bibinfo{year}{2022}\natexlab{}.
\newblock \showarticletitle{Causal inference for social network data}.
\newblock \bibinfo{journal}{\emph{J. Amer. Statist. Assoc.}} (\bibinfo{year}{2022}), \bibinfo{pages}{1--15}.
\newblock


\bibitem[Patacchini et~al\mbox{.}(2017)]%
        {patacchini-jcbo17}
\bibfield{author}{\bibinfo{person}{Eleonora Patacchini}, \bibinfo{person}{Edoardo Rainone}, {and} \bibinfo{person}{Yves Zenou}.} \bibinfo{year}{2017}\natexlab{}.
\newblock \showarticletitle{Heterogeneous peer effects in education}.
\newblock \bibinfo{journal}{\emph{Journal of Economic Behavior \& Organization}}  \bibinfo{volume}{134} (\bibinfo{year}{2017}), \bibinfo{pages}{190--227}.
\newblock


\bibitem[Pearl(2009)]%
        {pearl-book09}
\bibfield{author}{\bibinfo{person}{Judea Pearl}.} \bibinfo{year}{2009}\natexlab{}.
\newblock \bibinfo{booktitle}{\emph{Causality}}.
\newblock \bibinfo{publisher}{Cambridge university press}.
\newblock


\bibitem[Qu et~al\mbox{.}(2021)]%
        {qu-arxiv21}
\bibfield{author}{\bibinfo{person}{Zhaonan Qu}, \bibinfo{person}{Ruoxuan Xiong}, \bibinfo{person}{Jizhou Liu}, {and} \bibinfo{person}{Guido Imbens}.} \bibinfo{year}{2021}\natexlab{}.
\newblock \showarticletitle{Efficient Treatment Effect Estimation in Observational Studies under Heterogeneous Partial Interference}.
\newblock \bibinfo{journal}{\emph{arXiv preprint arXiv:2107.12420}} (\bibinfo{year}{2021}).
\newblock


\bibitem[S{\"a}vje(2024)]%
        {savje-biometrika24}
\bibfield{author}{\bibinfo{person}{Fredrik S{\"a}vje}.} \bibinfo{year}{2024}\natexlab{}.
\newblock \showarticletitle{Causal inference with misspecified exposure mappings: separating definitions and assumptions}.
\newblock \bibinfo{journal}{\emph{Biometrika}} \bibinfo{volume}{111}, \bibinfo{number}{1} (\bibinfo{year}{2024}), \bibinfo{pages}{1--15}.
\newblock


\bibitem[Shalit et~al\mbox{.}(2017)]%
        {shalit-icml17}
\bibfield{author}{\bibinfo{person}{Uri Shalit}, \bibinfo{person}{Fredrik~D Johansson}, {and} \bibinfo{person}{David Sontag}.} \bibinfo{year}{2017}\natexlab{}.
\newblock \showarticletitle{Estimating individual treatment effect: generalization bounds and algorithms}. In \bibinfo{booktitle}{\emph{International Conference on Machine Learning}}. PMLR, \bibinfo{pages}{3076--3085}.
\newblock


\bibitem[Shi et~al\mbox{.}(2019)]%
        {shi-neurips19}
\bibfield{author}{\bibinfo{person}{Claudia Shi}, \bibinfo{person}{David Blei}, {and} \bibinfo{person}{Victor Veitch}.} \bibinfo{year}{2019}\natexlab{}.
\newblock \showarticletitle{Adapting neural networks for the estimation of treatment effects}.
\newblock \bibinfo{journal}{\emph{Advances in neural information processing systems}}  \bibinfo{volume}{32} (\bibinfo{year}{2019}).
\newblock


\bibitem[Tchetgen et~al\mbox{.}(2020)]%
        {tchetgen-arxiv20}
\bibfield{author}{\bibinfo{person}{Eric J~Tchetgen Tchetgen}, \bibinfo{person}{Andrew Ying}, \bibinfo{person}{Yifan Cui}, \bibinfo{person}{Xu Shi}, {and} \bibinfo{person}{Wang Miao}.} \bibinfo{year}{2020}\natexlab{}.
\newblock \showarticletitle{An introduction to proximal causal learning}.
\newblock \bibinfo{journal}{\emph{arXiv preprint arXiv:2009.10982}} (\bibinfo{year}{2020}).
\newblock


\bibitem[Tran and Zheleva(2022)]%
        {tran-aaai22}
\bibfield{author}{\bibinfo{person}{Christopher Tran} {and} \bibinfo{person}{Elena Zheleva}.} \bibinfo{year}{2022}\natexlab{}.
\newblock \showarticletitle{Heterogeneous Peer Effects in the Linear Threshold Model}.
\newblock \bibinfo{journal}{\emph{Proceedings of the AAAI Conference on Artificial Intelligence}} (\bibinfo{year}{2022}).
\newblock


\bibitem[Ugander et~al\mbox{.}(2013)]%
        {ugander-kdd13}
\bibfield{author}{\bibinfo{person}{Johan Ugander}, \bibinfo{person}{Brian Karrer}, \bibinfo{person}{Lars Backstrom}, {and} \bibinfo{person}{Jon Kleinberg}.} \bibinfo{year}{2013}\natexlab{}.
\newblock \showarticletitle{Graph cluster randomization: Network exposure to multiple universes}. In \bibinfo{booktitle}{\emph{Proceedings of the 19th ACM SIGKDD international conference on Knowledge discovery and data mining}}. \bibinfo{pages}{329--337}.
\newblock


\bibitem[Watts and Strogatz(1998)]%
        {watts-nature98}
\bibfield{author}{\bibinfo{person}{Duncan~J Watts} {and} \bibinfo{person}{Steven~H Strogatz}.} \bibinfo{year}{1998}\natexlab{}.
\newblock \showarticletitle{Collective dynamics of ‘small-world’networks}.
\newblock \bibinfo{journal}{\emph{nature}} \bibinfo{volume}{393}, \bibinfo{number}{6684} (\bibinfo{year}{1998}), \bibinfo{pages}{440--442}.
\newblock


\bibitem[Wu et~al\mbox{.}(2025)]%
        {wu-iclr25}
\bibfield{author}{\bibinfo{person}{Anpeng Wu}, \bibinfo{person}{Haiyi Qiu}, \bibinfo{person}{Zhengming Chen}, \bibinfo{person}{Zijian Li}, \bibinfo{person}{Ruoxuan Xiong}, \bibinfo{person}{Fei Wu}, {and} \bibinfo{person}{Kun Zhang}.} \bibinfo{year}{2025}\natexlab{}.
\newblock \showarticletitle{Causal Graph Transformer for Treatment Effect Estimation Under Unknown Interference}. In \bibinfo{booktitle}{\emph{The Thirteenth International Conference on Learning Representations}}.
\newblock


\bibitem[Xu et~al\mbox{.}(2018)]%
        {xu-iclr18}
\bibfield{author}{\bibinfo{person}{Keyulu Xu}, \bibinfo{person}{Weihua Hu}, \bibinfo{person}{Jure Leskovec}, {and} \bibinfo{person}{Stefanie Jegelka}.} \bibinfo{year}{2018}\natexlab{}.
\newblock \showarticletitle{How Powerful are Graph Neural Networks?}. In \bibinfo{booktitle}{\emph{International Conference on Learning Representations}}.
\newblock


\bibitem[Yuan et~al\mbox{.}(2021)]%
        {yuan-www21}
\bibfield{author}{\bibinfo{person}{Yuan Yuan}, \bibinfo{person}{Kristen Altenburger}, {and} \bibinfo{person}{Farshad Kooti}.} \bibinfo{year}{2021}\natexlab{}.
\newblock \showarticletitle{Causal Network Motifs: Identifying Heterogeneous Spillover Effects in A/B Tests}. In \bibinfo{booktitle}{\emph{Proceedings of the Web Conference 2021}}. \bibinfo{pages}{3359--3370}.
\newblock


\bibitem[Zhao et~al\mbox{.}(2024)]%
        {zhao-arxiv22}
\bibfield{author}{\bibinfo{person}{Ziyu Zhao}, \bibinfo{person}{Yuqi Bai}, \bibinfo{person}{Ruoxuan Xiong}, \bibinfo{person}{Qingyu Cao}, \bibinfo{person}{Chao Ma}, \bibinfo{person}{Ning Jiang}, \bibinfo{person}{Fei Wu}, {and} \bibinfo{person}{Kun Kuang}.} \bibinfo{year}{2024}\natexlab{}.
\newblock \showarticletitle{Learning Individual Treatment Effects under Heterogeneous Interference in Networks}.
\newblock \bibinfo{journal}{\emph{ACM Trans. Knowl. Discov. Data}} \bibinfo{volume}{18}, \bibinfo{number}{8}, Article \bibinfo{articleno}{199} (\bibinfo{date}{Aug.} \bibinfo{year}{2024}), \bibinfo{numpages}{21}~pages.
\newblock
\showISSN{1556-4681}
\urldef\tempurl%
\url{https://doi.org/10.1145/3673761}
\showDOI{\tempurl}


\end{thebibliography}
\appendix
\section{Appendix}
Source code and documentation are available at: \url{https://anonymous.4open.science/r/EgoNetGNN-8D5C/}

\subsection{Discussion}\label{ap-discussion}

\textbf{Societal impacts}. The implications of our work include identifying unit-level peer effects and discovering subpopulations with heterogeneous peer effects. The potential societal impacts could include the development of targeted interventions or the identification of policies that enhance desired outcomes in social networks.

\textbf{Plausibility of neighborhood interference assumption}. Neighborhood interference (Assumption 2 in Sec. \ref{sec:problem_setup}) is a common simplifying assumption and can be realistic in situations where peer interference is mediated by immediate neighbors or diminishes quickly for non-immediate neighbors. However, there could be some situations where interference could occur between peers beyond immediate neighbors. If we assume such interference is mediated via immediate neighbors, then we could model it by stacking multiple exposure mapping function learning layers, where the subsequent layers would summarize the exposures of neighbors. Another alternative is to use the K-hop ego network with edge existence and/or hop distance as additional node features. The former approach may be more scalable than the latter one because the K-hop neighborhood can grow rapidly. Ideas from recent works to infer unknown interference structure~\cite{wu-iclr25,lin-packdd24} could be adopted in conjunction with our approach of learning expressive peer exposure representations. While we assume a reliable network structure is provided as input, our experiments with noisy networks reveal that \ourmodel performs reliably well with imperfect data.

\textbf{Plausibility of unconfoundedness assumption}. Following existing work in the intersection of causal reasoning and representation learning~\cite{shalit-icml17,shi-neurips19,ma-kdd22,wu-iclr25}, we assume causal identification conditions are met and focus on expressive representation learning to mitigate model misspecification errors. Unconfoundness is a strong and untestable assumption and requires sufficiency of observed network contexts and expressiveness of their representation. While we assume the sufficiency of observed contexts, we make an effort to satisfy the expressiveness of representation by considering all network contexts, like node attributes, edge attributes, and network structure. If the presence of unobserved confounding cannot be ruled out, alternative causal identification approaches like proximal causal inference~\cite{tchetgen-arxiv20} or double negative controls~\cite{miao-strf24}, front-door criteria~\cite{pearl-book09}, and instrumental variables~\cite{angrist-jasa96} should be considered. Although a randomized experiment can remove unobserved confounding between unit treatments and the outcome, peer exposure conditions may not be randomized directly, and confounding could exist even for experiments unless the unconfoundedness assumption is made and observed network contexts are controlled for. So, an interesting future direction could be to explore alternative identification conditions.

\textbf{Scalability}. Although \ourmodel is more expressive, it has additional computational costs. A few ways to address large runtime and/or memory usage could be sampling ego networks to reduce the training set or sampling the neighborhood within a K-hop ego network. In Appendix A.8 (Table 6), our experiments with a randomly augmented network show that the performance does not degrade significantly for our method with the removal of edges. From an implementation point of view, we can parallelize our framework easily to exploit the power of GPUs. More specifically, there are two components in our framework. The feature mapping GNN takes the entire network at once to learn an embedding with an L-layer GNN. Subsequently, EgoNetGNN batches B nodes with their neighbor nodes and a mapping of which edges belong to which node in the batch. This batching can be parallelized to improve the overall efficiency.  

% \vspace{-1em}
\subsection{Related Work}\label{ap:related-work}
{Research in causal inference under interference has focused on estimating three main causal effects of interest, referred to as network effects: direct effects induced by a unit's own treatment, peer effects induced by treatment of other units, and total effects induced by both the unit's and others' treatment~\cite{hudgens-jasa08}. These network effects are estimated as average effects (e.g.,~\cite{arbour-kdd16,ugander-kdd13}) for the entire population or as heterogeneous effects (e.g.,~\cite{forastiere-asa21,bargagli-aas25}) for specific subpopulations or contexts. Our work focuses on heterogeneous peer effect estimation.
Most methods for estimating heterogeneous or individual-level causal effects under interference, including peer effects, assume peer exposure is binary~\cite{bargagli-aas25} or homogeneous, e.g., based on fraction of treated peers~\cite{jiang-cikm22,ogburn-jasa22,cai-cikm23,chen-icml24}. These methods assume a homogeneous or known exposure mapping function and focus on enhancing network effect estimation by adapting techniques like adversarial training~\cite{jiang-cikm22}, propensity score reweighting~\cite{cai-cikm23}, double machine learning~\cite{khatami-arxiv24}, \newchange{doubly robust estimation~\cite{leung-arxiv22}, targeted maximum likelihood estimate~\cite{ogburn-jasa22}, and targeted learning~\cite{chen-icml24}.}}
% {These methods adapt techniques like adversarial training~\cite{jiang-cikm22}, propensity score reweighting~\cite{cai-cikm23}, and doubly robust estimation via targeted learning~\cite{chen-icml24} for causal effect estimation in homogeneous interference settings.} 

{Recent research has looked into more complex functions of peer exposure, allowing for heterogeneous peer influence, in which different peers can have varying degrees of influence. Some of these works refer to heterogeneous peer influence as heterogeneous interference~\cite{qu-arxiv21,zhao-arxiv22,lin-ecmlkdd23}. \citet{forastiere-asa21} considered peer exposure as a weighted fraction of treated peers using known edge attributes as weights.  \citet{lin-ecmlkdd23} consider heterogeneity due to multiple entities types and \citet{qu-arxiv21} considered heterogeneity due to known node attributes for defining peer exposure. \citet{tran-aaai22} studied peer effect estimation with linear threshold peer exposure model but different unit-level threshold could be vary for different units capturing heterogeneous susceptibilities to the influence. ~\citet{zhao-arxiv22} used attention weights derived based on the similarities of the units' covariates to determine peer exposure as the weighted sum of treated peers. \citet{yuan-www21} capture peer exposure with features based on counts of different causal network motifs, i.e., recurrent subgraphs in a unit's ego network with treatment assignments as attributes. \citet{ma-aistats21} consider homogeneous peer exposure based on fraction of treated peers but they summarize the covariates of treated peers using a graph neural network (GNN) to capture heterogeneous contexts involving treatment assignments. Unlike our work, none of these studies has explicitly studied the issue of automatically learning the exposure mapping functions to define peer exposure representation while capturing the underlying influence mechanisms.}

\citet{ma-aistats21} learn heterogeneous contexts based on peer treatments but not the exposure mapping function or the peer exposure representation.
\citet{zhao-arxiv22} obtain single-dimension peer exposure embedding using a weighted sum of treated peers with attention weights derived from the cosine similarity of feature embeddings. Although \citet{zhao-arxiv22} use attention weights to define peer exposure, they assume a specific exposure mapping function, and it cannot adapt according to the underlying peer influence mechanism.
\citet{adhikari-mlj24} use GNNs to learn peer exposure embedding by addressing unknown peer influence mechanisms, but their scope is limited to direct effect estimation, i.e., the effect of a unit's own treatment. Specifically, \citet{adhikari-mlj24} learn a multi-dimensional peer exposure embedding using a weighted fraction of treated peers with feature embeddings and a second-order adjacency matrix as weights. \citet{ma-kdd22} employ similar method like \citet{ma-aistats21} for hypergraphs to model heterogeneity due to model group interactions. The idea is to learn a summary function and representation equivalent to the exposure mapping function and peer exposure using a hypergraph convolution network and attention mechanism. However, they assume the learned representation is expressive enough to capture the underlying influence mechanism. In this work, we do not make such an assumption and evaluate how well the learned peer exposure representation captures the underlying influence mechanisms.

% Recently, graph neural networks (GNNs) have been widely utilized for estimating causal effects in networks~\cite{jiang-cikm22,cai-cikm23,chen-icml24,khatami-arxiv24}; however, their application has largely been confined to addressing confounding specific to networks (e.g., due to latent homophily, a tendency of similar units to be connected~\cite{cristali-neurips22}). Our work explores the potential of GNNs to learn exposure mapping functions with the goal of capturing underlying influence mechanisms due to local neighborhood structures. Prior research~\cite{xu-iclr18,chen-neurips20} on the expressiveness of GNNs has shown popular message-passing GNNs lack expressiveness to count subgraphs. On the other hand, counts of causal network motifs are rich features that could capture underlying influence mechanisms due to local neighborhood structure~\cite{yuan-www21}. Counting such subgraphs can be computationally expensive, and they may not be able to capture every local structure. We design \ourmodel to excel in counting attributed triangle subgraphs, enhancing its expressiveness to capture underlying mechanisms involving  neighborhood structure.

Neural networks (NNs)~\cite{shalit-icml17,im-arxiv21,shi-neurips19} and, recently, graph neural networks (GNNs)~\cite{jiang-cikm22,cai-cikm23,chen-icml24,khatami-arxiv24} have been widely utilized for end-to-end learning of \textit{feature mapping function} and \textit{counterfactual outcome model} or \textit{effect estimator}. 
A feature mapping function maps raw features to feature embedding to capture potential confounders and effect modifiers. A counterfactual outcome model~\cite{shalit-icml17,ma-aistats21} predicts counterfactual outcomes for different levels of treatment, while an effect estimator~\cite{shi-neurips19,chen-icml24} directly learns the causal effect of interest. Only a few studies have considered learning the exposure mapping function~\cite{mao-icassp25} or peer exposure embedding~\cite{adhikari-mlj24,zhao-arxiv22}. 
% \citet{adhikari-mlj24} learn a multi-dimensional peer exposure embedding using a weighted fraction of treated peers with feature embeddings and a second-order adjacency matrix as weights. \citet{zhao-arxiv22} obtain single-dimension peer exposure embedding using a weighted sum of treated peers with attention weights derived from the cosine similarity of feature embeddings.
% Recently, graph neural networks (GNNs) have been widely utilized for estimating causal effects in networks~\cite{jiang-cikm22,cai-cikm23,chen-icml24,khatami-arxiv24}, but their use has largely been confined to addressing confounding specific to networks (e.g., due to latent homophily, a tendency of similar units to be connected~\cite{cristali-neurips22}). 
% Our work explores the potential of GNNs to learn exposure mapping functions with the goal of capturing unknown underlying influence mechanisms including mechanisms involving local neighborhood structures.
\newchange{\citet{lin-packdd24} consider a setting with an unknown network and interference structure and propose an approach to first infer network structure and represent peer exposure for direct effect estimation. Unlike their work, our settings focus on peer effect estimation with observed network structure but unknown peer exposure mechanisms that manifest due to local neighborhood contexts.}

\newchange{\citet{savje-biometrika24} advocates for interpretable but possibly misspecified exposure mappings and characterizes causal estimation errors due to misspecified exposure mappings, but follow-up research~\cite{auerbach-arxiv24} has highlighted the importance of capturing underlying interference mechanisms in policymaking.} 
\newchange{More recently, \citet{mao-icassp25} have explored the use of GNNs with autoencoders and clustering to learn discrete exposure conditions and their probabilities, aiming to estimate overall causal effects in networks. Similarly, \citet{wu-iclr25} utilize GNNs with Transformers to model unknown interference from K-hop neighborhood. Their identifiability assumption relies on capturing unit and peer covariates, while our identifiability assumption relies on capturing all attributed network contexts, including structure and edge attributes. These works use off-the-shelf message passing GNNs (like GCN and GIN) and lack expressiveness to capture mechanisms involving local neighborhood structure.} Prior research~\cite{xu-iclr18,chen-neurips20} on the expressiveness of GNNs has shown that popular GNN architectures lack expressiveness to count subgraphs. On the other hand, counts of subgraphs like causal network motifs are rich features that could capture underlying influence mechanisms due to local neighborhood structure~\cite{yuan-www21}. Counting such subgraphs can be computationally expensive, and they may not be able to capture every local structure. We design \ourmodel to excel in counting attributed triangle subgraphs, enhancing its expressiveness to capture underlying mechanisms involving  neighborhood contexts.
% Recent works solving a diverse set of problems have {implicitly or explicitly} addressed heterogeneous peer influence (HPI) due to known contexts~\cite{qu-arxiv21,forastiere-jasa21} or specific contexts~\cite{yuan-www21,ma-kdd22,tran-aaai22,zhao-arxiv22,lin-ecmlkdd23}.
% Some of these works refer to HPI as heterogeneous interference~\cite{qu-arxiv21,zhao-arxiv22,lin-ecmlkdd23}.
% \citet{tran-aaai22} study peer contagion effects with homogeneous influence but different unit-level susceptibilities to the influence.
% \citet{yuan-www21} capture peer exposure with causal network motifs, i.e., recurrent subgraphs in a unit's ego network with treatment assignments. \citet{ma-kdd22} focus on addressing heterogeneous influence due to group interactions utilizing hypergraphs. \citet{zhao-arxiv22} deal with heterogeneity due to node attribute similarity using attention weights to estimate peer exposure and causal effects. \citet{lin-ecmlkdd23} consider heterogeneity due to multiple entities and relationships in networks. \citet{ma-aistats21} summarize the covariates of treated peers using a graph neural network (GNN) to learn a peer exposure embedding in addition to homogeneous peer exposure. \citet{ma-kdd22} employ similar method but for hypergraphs to model group interactions.

\subsection{\newchange{Causal Inference Assumptions and Identification of Peer Effects}}\label{ap-problem-setup}

A fundamental prerequisite for causal identification is the consistency assumption, which enables equivalence among counterfactual, interventional, and factual outcomes. 
\begin{assumption}[Consistency under interference]\label{asum:consistency}
    The underlying outcome generation is independent of the treatment assignment mechanisms (i.e., hypothetical, experimental, or natural). For a unit $v_i$, if $\rt_i=\pi_i$ and $\rvt_{-i}=\bm{\pi}_{-i}$, then $\ry_i(\rt_i=\pi_i, \rvt_{-i}=\bm{\pi}_{-i}) = \ry_i$.
\end{assumption}

Positivity is another standard assumption in causal inference that requires every unit $v_i$ to have non-zero probability of being assigned every possible unit treatment and peer exposure conditions.
\begin{assumption}[Positivity]\label{asum:pos}
    There is a non-zero probability of unit treatment and peer exposure conditions for all possible contexts $\rvc_i$, i.e., $\mathbb{P}(\rt_i, \rvp_{i}|\rvc_i) > 0$, for every level of $\rt_i$ and $\rvp_i$, where $\mathbb{P}$ is the probability density function.
\end{assumption}

The proof of Proposition \ref{prop:estimation} is as follows.
\begin{proof}
    Our causal estimand of interest (Eq. 2) is as follows:
    \begin{equation*}
    % \begin{split}
    \label{eq:peer_eff_exp_map}
    \delta_i(\bm{\pi}_{-i}, \bm{\pi}'_{-i}) = \mathbb{E}[\ry_i(\rt_i=\pi_i, \rvp_{i}=\phi_e(\bm{\pi}_{-i}, \gG, \rmX, \rmZ))| \rvc_i] - \mathbb{E}[\ry_i(\rt_i=\pi_i, \rvp_{i}=\phi_e(\bm{\pi}'_{-i}, \gG, \rmX, \rmZ)) | \rvc_i].
    % \end{split}
\end{equation*}
Due to unconfoundedness assumption (Assumption \ref{asum:unconfoundedness}), unit treatment and peer exposure conditions are independent of counterfactual outcome conditioned on network contexts $\rvc_i$. This allows us to rewrite the estimand as:
\begin{equation*}
    \begin{split}
    \label{eq:peer_eff_exp_map}
    \delta_i(\bm{\pi}_{-i}, \bm{\pi}'_{-i}) = \mathbb{E}[\ry_i(\rt_i=\pi_i, \rvp_{i}=\phi_e(\bm{\pi}_{-i}, \gG, \rmX, \rmZ))| \rt_i=\pi_i, \rvp_{i}=\phi_e(\bm{\pi}_{-i}, \gG, \rmX, \rmZ)), \rvc_i] - \\\mathbb{E}[\ry_i(\rt_i=\pi_i, \rvp_{i}=\phi_e(\bm{\pi}'_{-i}, \gG, \rmX, \rmZ)) | \rt_i=\pi_i, \rvp_{i}=\phi_e(\bm{\pi}_{-i}, \gG, \rmX, \rmZ)), \rvc_i].
    \end{split}
\end{equation*}
Here, Assumption \ref{assum:pre} ensures introducing new terms related to treatment and peer exposure in the conditional does not affect existing set of contexts because they are measured pre-treatment. Similarly, Assumption \ref{assum:neigh-int} makes the sufficiency of learned representation requirement in unconfoundedness assumption more plausible.
Next, the consistency assumption allows replacing the counterfactual outcome with observed outcome, i.e., 
\begin{equation*}
    \begin{split}
    \label{eq:peer_eff_exp_map}
    \delta_i(\bm{\pi}_{-i}, \bm{\pi}'_{-i}) = \mathbb{E}[\ry_i| \rt_i=\pi_i, \rvp_{i}=\phi_e(\bm{\pi}_{-i}, \gG, \rmX, \rmZ)), \rvc_i] - \mathbb{E}[\ry_i | \rt_i=\pi_i, \rvp_{i}=\phi_e(\bm{\pi}_{-i}, \gG, \rmX, \rmZ)), \rvc_i].
    \end{split}
\end{equation*}
Assumption \ref{assum:pre} also ensures consistency assumption is satisfied because the treatments are not mutable.
This estimation above is tractable from observational or experimental data because of positivity assumption and the causal effects can be identified. 
\end{proof}

\subsection{Theoretical Analyses of \ourmodel}\label{ap-theory}
\subsubsection{Preliminaries}

\begin{figure*}
    \centering
    \begin{minipage}{.38\linewidth}
        \centering
        \includegraphics[width=\linewidth]{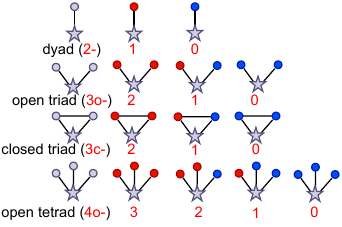}
        \caption{Example causal network motifs considered by \citet{yuan-www21}. Stars represent ego nodes and circles represent their peers. The red circles indicate treated nodes and blue circles indicate control nodes. The gray shapes indicate nodes that could either be treated or control. Here, the characters in red indicate a particular causal network motif (e.g., 3c-2 indicate closed triad with 2 treated peers).}
        \label{fig:motif-example}
    \end{minipage}% 
    \hfill
    \begin{minipage}{.57\linewidth}
        \includegraphics[width=\linewidth]{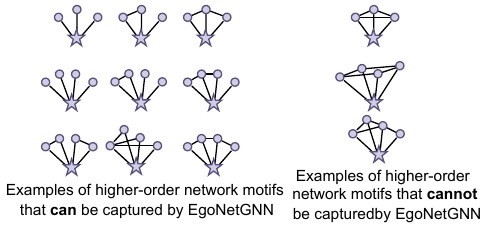}
        \caption{Examples of higher-order network motifs with four and five nodes. Stars represent ego nodes and circles represent their peers. The gray shapes indicate nodes with any treatment assignment. If the subgraph of a network motif, after removing edges connected to the ego node, forms a tree, then our model is expressive enough to capture the network motif and the corresponding causal network motifs. A network motif is a subgraph without any attributes, whereas a causal network motif is a subgraph that includes peer treatment assignments as attributes.}
        \label{fig:higher-order}
    \end{minipage}
    % \caption{Example (causal) network motifs. Stars represent ego nodes and circles represent their peers. The red circles indicate treated nodes and blue circles indicate control nodes. The gray shapes indicate nodes that could either be treated or control. A network motif is a subgraph without any attributes, whereas a causal network motif is a subgraph that includes peer treatment assignments as attributes. \sm{Is there any reason to put 4a and 4b put together? if not, may be we can simplify. also, the caption has mix of both a and b without separating}}
\end{figure*}
\textbf{Causal network motifs}. \citet{yuan-www21} proposed causal network motifs as important features to capture peer exposure accounting for local neighborhood conditions. Causal network motifs are attributed subgraphs with peer treatments as attributes. Figure \ref{fig:motif-example} shows four categories of causal network motifs: dyads, open triads, closed triads, and open tetrads. In the figure, stars represent ego nodes and circles represent their peers. The red circles indicate treated nodes and blue circles indicate control nodes. The gray shapes indicate nodes that could either be treated or control.

\textbf{Message passing graph neural networks (MPGNNs)}. The message-passing graph neural network (MPGNN) is a generic GNN model that incorporates several standard GNN architectures and relies on local aggregations of information within graphs~\cite{chen-neurips20}.  For a graph $G(V,E,\mathbf{X},\mathbf{Z})$, an MPGNN with $L$ layers is defined iteratively with aggregate function $AGG^l$ and update function $U^l$ as follows:
\begin{equation}
    h_i^l = U^l(h_i^{l-1}, AGG^l_{j \in \mathcal{N}_i}(\Theta^l(h_j^{l-1}, h_i^{l-1}, Z_{ij}))),
\end{equation}
where $\mathcal{N}_i$ denotes neighbors of unit $v_i$ and $\Theta^l$ denote learnable parameters like multi-layer perceptron. To obtain the hidden state at the $l^{th}$ layer, a local aggregation of the previous layer's hidden states ($h_j^{l-1}$ and $h_i^{l-1}$) and, optionally, edge attributes $Z_{ij}$ is performed and then combined with $h_i^{l-1}$. The hidden states are initialized as node attributes, i.e., $h^0_i=X_i$. Typically, in various GNN architectures, the update and aggregation functions are chosen as part of architecture design.

\textbf{Expressiveness of MPGNNs in counting substructures}. Here, we summarize the results obtained by ~\citet{chen-neurips20} that are relevant to our theoretical analysis. We list their findings after defining relevant concepts.
\begin{definition}[Subgraph]
    A \textit{subgraph} $G^{[S]}(V^{[S]}, E^{[S]}))$ of a graph $G(V,E)$ consists of subsets of its nodes, i.e., $V^{[S]} \subseteq V$ and edges, i.e., $E^{[S]}\subseteq E$.
\end{definition}
\begin{definition}[Induced subgraph]
    A \textit{induced subgraph} $G^{[S']}(V^{[S']}, E^{[S']})$ of a graph $G(V,E)$ consists of subset of its nodes, i.e., $V^{[S']}\subseteq V$ and all edges between nodes $V^{[S']}$, i.e., $E^{[S']}= E \cap V^{[S']}$.
\end{definition}
All induced subgraphs are subgraphs but reverse is not true. For example, all causal network motfis are induced subgraphs (and subgraphs) of the original graph. An open triad motif is a subgraph, but not an induced subgraph, of a closed triad motif.
\begin{definition}[Star-shaped pattern]
    A pattern $G^{[P]}(V^{[P]}, E^{[P]})$ is a star-shaped pattern if it can be represented by a tree structure.
\end{definition}
\begin{definition}[Connected pattern]
    A pattern $G^{[P]}(V^{[P]}, E^{[P]})$ is a connected pattern if it \textbf{cannot} be represented by a tree structure.
\end{definition}
For example, a closed triad motif is a connected pattern and dyads, open triads, and open tetrads are star-shaped patterns.

\citet{chen-neurips20} obtain the following results on the expressiveness of MPGNNS for counting substructures.

\textbf{Corollary 3.4.}~\cite{chen-neurips20} MPGNNs cannot \textit{induced-subgraph-count} any \textit{connected pattern} with 3 or more nodes.

\textbf{Theorem 3.5.}~\cite{chen-neurips20} MPGNNs can perform \textit{subgraph-count} of \textit{star-shaped patterns}.

\subsubsection{Expressiveness of \ourmodel}
Here, we demonstrate that standard MPGNNs lack the expressiveness to capture closed triad motifs, and our model addresses this limitation.

Without loss of generality, assume node attributes for each node $v_i$ are $<1, T_i>$ and constant edge attributes $<1>$.

\begin{definition}[Expressiveness in counting causal network motifs]
    Let $\mathcal{G}$ be a space of graphs. A representation by an MPGNN $f$ is expressive in counting causal  network motif $G^{[P]}$ if, for all ego networks $G^{[1]},G^{[2]} \in \mathcal{G}$, distinct counts, i.e., $C_I(G^{[1]}, G^{[P]}) \neq C_I(G^{[2]},G^{[P]})$, get distinct representations, i.e., $f(G^{[1]})\neq f(G^{[2]})$, where $C_I$ returns induced-subgraph-count of pattern $G^{[P]}$.
\end{definition}
% \sm{you don't need to prove Lemma 1, in fact Lemma 1 should not be there - it has been proven - just go straight to the point}
% \begin{lemma}[Expressiveness of MPGNNs]
%     MPGNNs can capture dyad, open triad ,and open tetrad causal network motifs but lack the expressiveness to capture closed triad causal network motifs.
% \end{lemma}
% \begin{proof}
%     The proof directly follows from \citet{chen-neurips20}'s Theorem 3.5. and Corollary 3.4. The dyad, open triad, and open tetrad causal network motifs are star-shaped patterns and these  patterns can be counted by MPGNNs (Theorem 3.5.).
%     The closed triad causal network motifs are connected patterns of three nodes and these patterns cannot be counted by MPGNNs (Corollary 3.4.).
% \end{proof}

\begin{proposition}[Expressiveness of \ourmodel]\label{prop:exp}
    \ourmodel is expressive enough to capture all dyad, open triad, closed triad, and open tetrad causal network motifs.
\end{proposition}
\begin{proof}
We proceed the proof by dividing the statement into following two claims.
\begin{claim}
    \ourmodel is as expressive as standard MPGNN in capturing dyad, open triad, and open tetrad causal network motifs.
\end{claim}
\begin{proof}
    The dyad, open triad, and open tetrad causal network motifs are star-shaped patterns, and these patterns can be counted by standard MPGNNs (\citet{chen-neurips20}'s Theorem 3.5.). Our model employs MPGNN (refer Eq. \ref{eq:node-agg} and Figure \ref{fig:framework}) on a transformed graph, where all edges connected to the ego node are removed, and the corresponding edge attributes from the removed edges are included as node attributes in the transformed graph. We need to show that this transformation preserves the expressiveness to capture dyad, open triad, and open tetrad causal network motifs. The dyad, open triad, and open tetrad causal network motifs are transformed into subgraphs with isolated one, two, and three nodes, respectively, in the transformed ego network. MPGNN in the transformed graph can perform a subgraph count of patterns with k isolated nodes because they are subgraphs of star-shaped patterns with an empty set of edges. Furthermore, the addition of new attributes does not affect the expressiveness because these attributes are added as additional feature dimensions. Hence, our model is as expressive as standard MPGNN for capturing dyad, open triad, and open tetrad causal network motifs.
\end{proof}
\begin{claim}
    \ourmodel also captures closed triad causal network motifs.
    % \sm{I would just stop here} that cannot be captured by standard MPGNNs.
\end{claim}
\begin{proof}
The closed triad causal network motifs are connected patterns of three nodes and these patterns cannot be counted by standard MPGNNs (\citet{chen-neurips20}'s Corollary 3.4.).
% Suppose, for contradiction, \ourmodel cannot capture closed triad causal network motifs.
Due to the construction of the ego network, all the edges with the ego node are removed, and the closed triads are transformed to dyads in the transformed ego network. These dyads can be counted by node aggregation (refer Eq. \ref{eq:node-agg}), which is an MPGNN employed in the ego network.
% \sm{here you have to make it clear that MPGNN is employed within \ourmodel,because people migght be confused about usage of MPGNN's in the next sentence}.
Therefore, \ourmodel captures closed triad causal network motifs.
% \sm{I would stop here} that cannot be captured by standard MPGNNs.
\end{proof}
\end{proof}

% \sm{what is the purpose of this paragraph? Ans: To show we are not just limited to four causal network motifs and capture more complex motifs.}
\textbf{Higher-order causal network motifs and attributed causal network motifs.} Here, we show how our model is superior to the approach of counting predetermined causal network motifs by discussing \ourmodel's ability to capture relevant causal network motifs including higher-order and attributed causal network motifs. Proposition \ref{prop:exp} showed our model is as expressive as the approach of counting predetermined causal network motifs considered by \citet{yuan-www21}. In general, if the subgraph of a network motif, after removing edges connected to the ego node, forms a tree, then \ourmodel is expressive enough to capture the network motif and the corresponding causal network motifs. Figure \ref{fig:higher-order} depicts some examples of higher-order motifs with four and five nodes. \ourmodel, with depths of $L=2$ and $L=3$ (refer Eq. \ref{eq:node-agg}), is expressive enough to capture most higher-order motifs with four and five nodes, respectively. Only if the network motifs consist of a cycle without the involvement of the ego node, then \ourmodel is not expressive enough to capture it. Furthermore, compared to predetermined causal network motifs, \ourmodel can accommodate motifs with additional node and edge attributes. Incorporating node and edge attributes will not reduce the expressiveness of counting original causal network motifs because these attributes are added as additional feature dimensions.

\subsubsection{Time complexity of \ourmodel-TARNet}\label{ap-theory-time}
Typically, the complexity of a standard MPGNN (e.g. GCN), is $O(NLF^2 + L|E|F)$, where $N$,$|E|$, $L$, and $F$ are the number of nodes, edges, GNN layers, and the dimensionality of feature embeddings, respectively~\cite{blakely-github2021}. In our model, the feature mapping MPGNN (refer to Eq. \ref{eq:fmap}) has the time complexity of $O(d_{\Theta}N{F_x}^2)$ for ego feature embedding module $\Theta_0(\mathbf{X_i})$, where $d_{\Theta}$ is the depth of MLP and $F_x$ is the dimensionality of node feature embedding, and $O(Ld_{\Theta}|E|F^2 + L|E|F)$ for peer feature embedding and aggregation, where $F = F_x + F_z$ is the dimensionality of node and edge feature embeddings. For node aggregation (refer to Eq. \ref{eq:node-agg}), we extract ego network for each node and perform neighborhood aggregation. Therefore, the time complexity is $O(NL|\Bar{E}_{max}|F)$, where $|\Bar{E}_{max}|$ is the number of maximum edges in the ego network. For subsequent masking and exposure encoding MLP, the time complexity is $O(Nd_{MLP}|\Bar{E}_{max}|F^2)$, where $d_{MLP}$ is the depth considering overall MLPs.

Assuming a single-layer MPGNN with $F<<N<|E|$, for simplicity, a standard MPGNN scales linearly with the number of edges, i.e., $O(|E|)$ or $O(N\times avg(D))$, where $avg(D)$ is the average degree.
Similarly, for \ourmodel the time complexity simplifies to $O(N \times |\Bar{E}_{max}|)$. In the worst case, $ |\Bar{E}_{max}|=max(D)^2$, where $max(D)$ is the maximum degree in the network $G(V,E)$. However, since networks are generally sparse, the approximate runtime complexity for networks with uniform degree (e.g., Watts Strogatz network or Stochastic Block Model network) is $O(N \times P_e \times avg(D)^2)$, where $P_e$ is density of edges. So, our method is approximately $P_e \times avg(D)$ times more computationally expensive than standard MPGNNs. On the other hand, the time complexity for counting predetermined causal network motifs with $K$ nodes is $O(N max(D)^{K-1})$, assuming access to O(1) adjacency set and adjacency matrix. This approach scales poorly with higher-order motifs and \ourmodel mitigates the problem by capturing most higher-order motifs with the same computational cost.

\subsubsection{Counterfactual outcome prediction error bounds for \ourmodel}\label{ap-theory-error}
Our work utilizes \citet{shalit-icml17}'s TARNet and CFR estimators, adapted to network settings, for estimating heterogeneous peer effects in both observational and experimental data. Their analysis shows the $PEHE$ metric is bounded by factual ($F$), i.e., supervised learning and counterfactual ($CF$) prediction error, i.e., $ \epsilon_{PEHE} ( \hat{f_y} ) \le 2( \epsilon_{CF} ( \hat{f_y}) + \epsilon_F ( \hat{f_y} ) - 2 \sigma^2_{\ry}) $, where $ \sigma^2_{\ry}$ is the variance of the outcome. These prediction errors or biases incorporate \citet{savje-biometrika24}'s definition of exposure mapping specification errors along with feature representation errors and outcome prediction errors.
% By learning a relevant exposure mapping function, our approach reliably reduces factual prediction error. It performs better than the baselines for reducing the counterfactual prediction error (due to informative peer exposure, Table 2.1 above). 

Moreover, \citet{shalit-icml17} show that the bound for counterfactual prediction error (which cannot be measured in the real world) depends on the Integral Probability Metric (IPM) measure of distance between treatment and control group distribution, which implies $\epsilon_{PEHE}(\hat{f_y})\le 2(\epsilon_F^{\rt_i=1}(\hat{f_y}) + \epsilon_F^{\rt_i=0}(\hat{f_y}) + \alpha IPM(\{\vh^{emb}_i: t_i=1\}, \{\vh^{emb}_i: t_i=0\}) - 2\sigma^2_Y)$, where $t_i=\pi_i$ denotes conditioning,  $\vh^{emb}_i = \Theta_{emb}(\hat{\vc_i}||\hat{\bm{\rho}_i})$, and $||$ denotes concatenation. 
% The IPM metric is approximated by metrics like Wasserstein distance between embeddings.
To study how misspecification errors of EgoNetGNN propagate to the factual prediction error, we can substitute the oracle values and estimated values (denoted with hat) and further decompose the errors by using sequential error decomposition trick, i.e.,
$$\epsilon_F^{\rt_i=\pi_i}(\hat{f}_y) = \mathbb{E}[(\hat{\ry_i}-\ry_i)^2]$$
$$\hat{\ry_i} - \ry_i =  \hat{f_y}(\pi_i, \hat{\bm{\rho}}_i,\hat{\rvc_i}) - f_y(\pi_i, \bm{\rho}_i,\rvc_i)$$
$\hat{\ry_i} - \ry_i = \epsilon_y + \epsilon_e + \epsilon_f$, where $\epsilon_y$ captures error due to learned outcome prediction module using learned representations, i.e.,
$$\epsilon_y := \hat{f_y}(\pi_i, \hat{\bm{\rho}}_i,\hat{\rvc_i}) - f_y(\pi_i, \hat{\bm{\rho}}_i,\hat{\rvc_i}),$$
$\epsilon_e$ captures error due to exposure mapping misspecification using learned feature representation but true outcome prediction module, i.e.,
$$\epsilon_e:=f_y(\pi_i, \hat{\bm{\rho}}_i,\hat{\rvc_i}) - f_y(\pi_i, \bm{\rho}_i,\hat{\rvc_i}),$$
and, finally, $\epsilon_f$ captures error due to feature mapping misspecification but true exposure and outcome prediction function, i.e.,
$$\epsilon_f:=f_y(\pi_i, \bm{\rho}_i,\hat{\rvc_i}) - f_y(\pi_i, \bm{\rho}_i,\rvc_i).$$

By plugging these decomposed errors in the factual prediction loss, we get,
$$\epsilon_F^{\rt_i=\pi_i}(\hat{f}_y) = \mathbb{E}[(\epsilon_y+\epsilon_e+\epsilon_f)^2]$$
$$=\E[\epsilon_y^2]+\E[\epsilon_e^2]+\E[\epsilon_f^2] + 2(\E[\epsilon_y \epsilon_e]+\E[\epsilon_e \epsilon_f] + \E[\epsilon_f \epsilon_y]).$$

By automatically learning relevant exposure mapping function, we aim to directly minimize the error terms involving $\epsilon_e$ and the downstream error $\epsilon_y$. Other estimators (e.g., Doubly robust or orthogonal learning after handling unknown exposure mapping function) can be employed in future work for more tight error bounds.

\subsection{Dataset Generation} \label{ap-dataset}
For the Barabasi Albert (BA) model, the preferential attachment parameter $m \in [1, 5, 10]$ is used to generate sparse to dense networks, where a new node connects to $p_{ba}$ existing nodes to form the network. For the Watts Strogatz (WS) model, we set mean degree parameters $k \in \{0.002N, 0.005N, 0.01N\}$ with fixed rewiring probability of $0.5$, similar to prior works~\cite{yuan-www21,adhikari-mlj24}. For the Stochastic Block Model (SBM) model, we use the number of blocks parameters $b \in \{500, 200, 100\}$ with randomly generated edge probabilities within and across communities. We also use two real-world social networks BlogCatalog and Flickr with more realistic topology and attributes to generate treatments and outcomes. We use LDA~\cite{blei-jmlr03} to reduce the dimensionality of raw features to $50$.

\textbf{Treatment model}.
The treatment assignments could depend on the unit's covariates as well as peer covariates and some edge attribute.
We generate treatment $T_i$ for a unit $v_i$ as $T_i \sim \theta \big(a(\tau_c\mathbf{W}_T \times \frac{\sum_{j \in \mathcal{N}_i} \mathbf{X^c}_j }{\sum_{j \in \mathcal{N}_i} Z^c_{ij}}) + (1-\tau_c)\mathbf{W}_T \cdot \mathbf{X^c}_i\big)$, where $\theta$ denotes Bernoulli distribution, $a:\mathbb{R}\mapsto [0,1]$ is an activation function, $\tau_c\in[0,1]$ controls spillover influence from unit $v_i$'s peers, $\mathbf{X^c} \subset \mathbf{X}$ is a subset of node attributes, $Z^c \in \mathbf{Z}$ is an edge attribute, and $\mathbf{W}_T$ is a weight matrix.

\textbf{Outcome model}. The outcomes depend on unit's treatment, peer treatments based on the local neighborhood condition, the confounders, and the effect modifiers. We generate outcome $Y_i$ for a unit $v_i$ as:
\begin{equation}
\begin{split}
    &Y_i = (\delta_{exp} + \delta_{em}\times T_i) \times \phi_{e}(G,\mathbf{X},\mathbf{Z},T_{-i}) + \\
    &(\tau_d  + \tau_{em} \times \phi_{em}(G,\mathbf{X},\mathbf{Z})) \times T_i + g(\mathbf{X_c}, Z_c, G) + \epsilon.
\end{split}
\label{eq:outcome}
\end{equation}
Here, the first term $(\delta_{exp} + \delta_{em}\times T_i) \times \phi_{e}(G,\mathbf{X}, \mathbf{Z},T_{-i})$ captures peer effects, where $\phi_{e}(G,\mathbf{X},\mathbf{Z},T_{-i})$ captures true peer exposure that depends on local neighborhood condition (e.g., the number of mutual connections between treated peers and ego unit or attribute similarity) and $\delta_{exp}$ and $\delta_{em}$ are coefficients controlling magnitude/direction of peer effects. The term $g(\mathbf{X_c}, Z_c, G)$ captures confounding and $\epsilon \sim \mathcal{N}(0,1)$ is random noise. The remaining term captures direct effect due to unit's own treatment with effect modification by some contexts. For semi-synthetic data, to generate heterogeneous peer effects, we use additional effect modification due to a unit's covariates, i.e., $\delta_{em}\times T_i\times \phi_v(\mathbf{X}_{em}),$ where $\mathbf{X}_{em} \subset \mathbf{X}$ and $\phi_v$ is a weighted mean function with randomly generated weights. Please refer to the source code in anonymous repository for detailed implementation of data generation.

\subsection{Additional Experimental Settings} \label{ap-hyperparams}

\textbf{Model implementation, hyperparameters, and model selection}. For the experiments, we choose $\lambda_{bal}=0.01$ for encouraging balanced representation and L1 loss regularization coefficient $\lambda_{L1}=1$ for encouraging invariance to irrelevant mechanism. We set the output embedding dimension of exposure encoder MLP to $3$ giving 6-dimensional peer exposure representation. We use $1-layer$ deep MPGNNs for feature and exposure mapping functions.
Moreover, we perform grid search hyperparameter tuning by varying GNN learning rate $\{0.1, 0.04, 0.02, 0.01\}$, and setting TARNet learning rate to $0.01$. We use Adam optimizer with weight decay of $10^{-5}$ and the learning rate is decayed by $50\%$ after $50$ epochs. A 20\% held-out dataset is used for model selection, where model with lowest outcome prediction loss $L_{Y_i}$ is chosen for reporting. We employ model checkpointing every other epoch to select the best performing model in a total of $100$ epochs. Our implementation is similar to \citet{adhikari-mlj24}'s INE-TARNet (also known as IDE-Net in original paper) in terms of MLP with residual network architecture, parameter tuning and model selection, and data generation.

The baselines INE-TARNet and GNN-TARNet-Motifs are also tuned similarly to our method by conducting grid search of the GNN's learning rate with $\{0.2, 0.02\}$ and variance smoothing regularization hyperparameter with $\{0.1, 1\}$, keeping TARNet's learning rate $0.02$ and other hyperparameters default. DWR is calibrated for $5$ epochs to balance representation. For other baselines, we use default hyperparameters.

\textbf{Implementation of baselines}. We use publicly available code shared for the baselines INE-TARNet~\cite{adhikari-mlj24}, TNet~\cite{chen-icml24}, NetEstimator~\cite{jiang-cikm22}, and CauGramer~\cite{wu-iclr25}. We adapt the code provided by authors to extend it for peer effect estimation for AEMNet~\cite{mao-icassp25}. We implement 1GNN-HSIC~\cite{ma-aistats21} and DWR~\cite{zhao-arxiv22} ourselves following the paper as closely as possible. GNN-TARNet-MOTIFS is available as a baseline of INE-TARNet. 

\textbf{Computational resources}. All the experiments are performed in a machine with the following resources.
\begin{itemize}
    \item CPU: AMD EPYC 7662 64-Core Processor (128 CPUs)
    \item Memory: 256 GB RAM
    \item Operating system: Ubuntu 20.04.4 LTS
    \item GPU: NVIDIA RTX A5000 (24 GB)
    \item CUDA Version: 11.4
\end{itemize}
As discussed in Section \ref{ap-theory}, the runtime of computation depends on the number of nodes and the number of edges in the ego networks along with the feature dimension. Here, we report execution time per iteration for training, evaluating, and checkpointing our model for synthetic and semi-synthetic network data. For the Barabasi Albert network with 3000 nodes, which is sparser, it takes approximately 2.1 seconds per iteration, whereas, for a Stochastic Block Model (SBM) with 3000 nodes, which is denser, it takes approximately 3.3 seconds per iteration. For the BlogCatalog network with 5196 nodes and 50-dimensional features, it takes around 5.7 seconds per iteration.

\subsection{Synthetic Data Experiments and Results}\label{ap-result-synthetic}
Figures \ref{fig:cc} to \ref{fig:attr-sim} show the performance of our method and baselines for three synthetic networks when the underlying peer exposure mechanisms depend on clustering coefficient, connected components, number of mutual connections, tie strengths, and attribute similarity. The results discussed in the main paper apply to additional peer exposure mechanisms and data generation conditions.
\begin{figure*}[!ht]
    \centering
    \includegraphics[width=\linewidth]{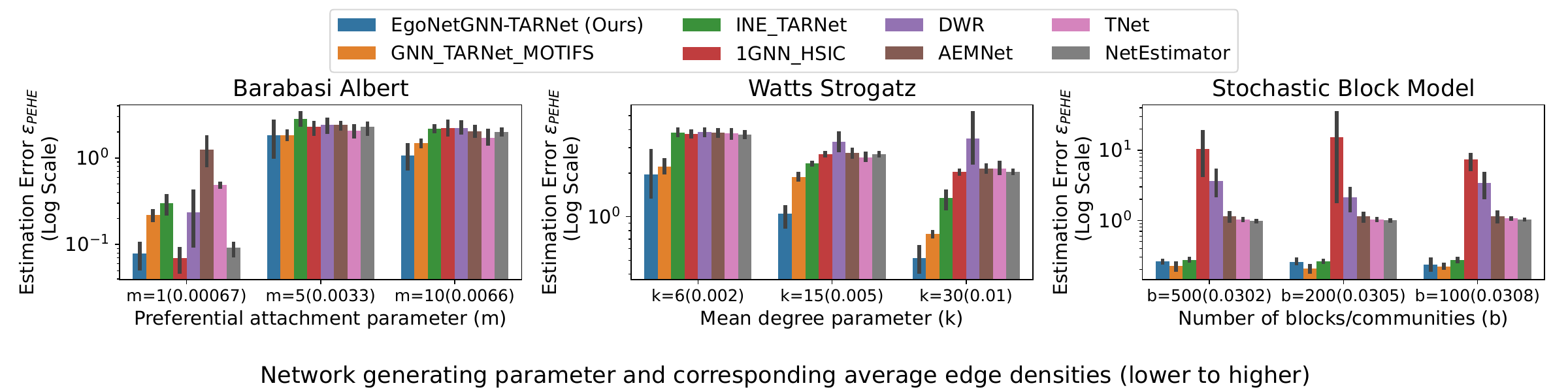}
    \caption{{Peer effect estimation error when true peer exposure depends on clustering coefficient among treated peers. Our method is better than or competitive to baseline using predetermined causal network motif counts when the underlying peer exposure mechanism can be explained by causal network motif counts.}}
    \label{fig:cc}
\end{figure*}
\begin{figure*}[!ht]
    \centering
    \includegraphics[width=\linewidth]{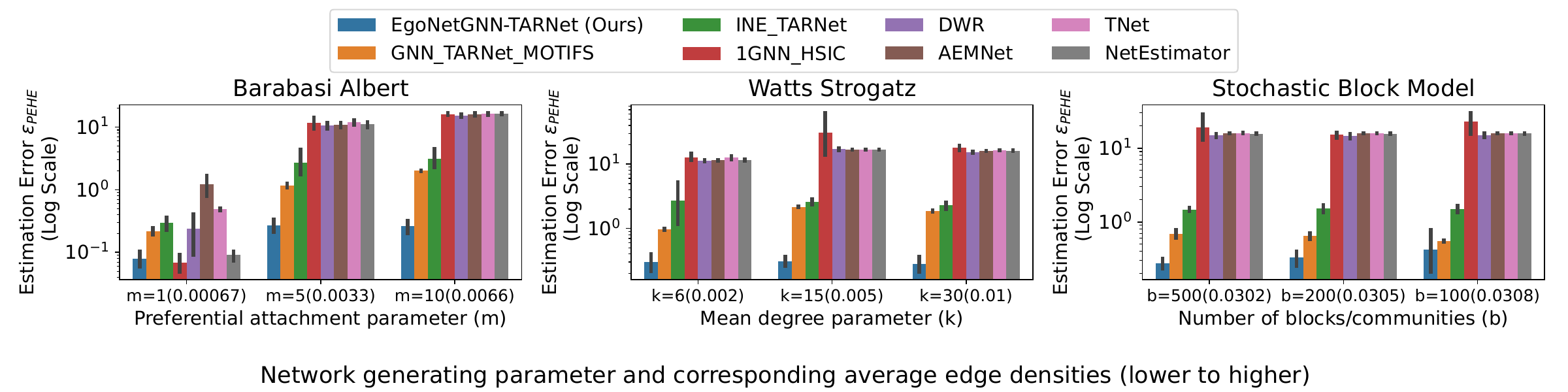}
    \caption{{Peer effect estimation error when true peer exposure depends on number of mutual connections with the ego. Our method significantly outperforms all baselines showing its capability to count closed triad network motifs (i.e., triangle substructures) in the ego network.}}
    \label{fig:mut-frns}
\end{figure*}
\begin{figure*}[!ht]
    \centering
    \includegraphics[width=\linewidth]{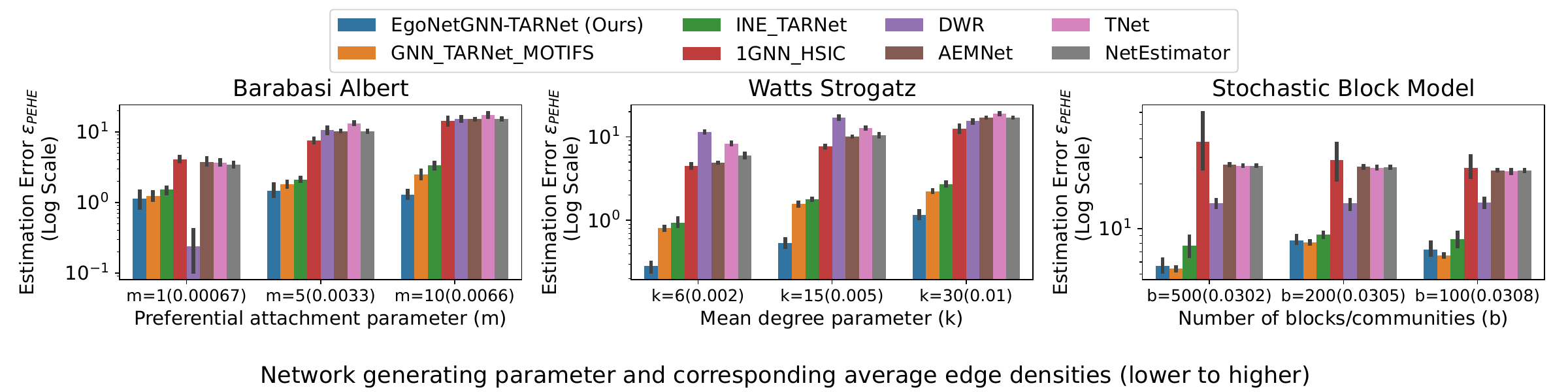}
    \caption{{Peer effect estimation error when true peer exposure depends on connected components among treated peers. Our method performs well compared to all baselines when underlying peer exposure mechanism cannot be explained totally with causal network motif structures only.}}
    \label{fig:con-comp}
\end{figure*}
\begin{figure*}[!ht]
    \centering
    \includegraphics[width=\linewidth]{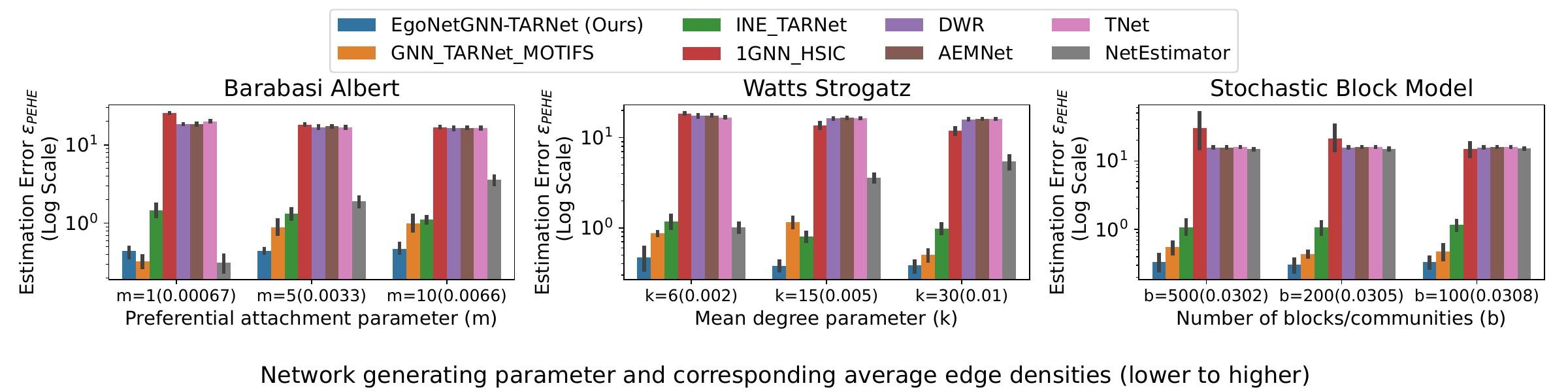}
    \caption{{Peer effect estimation error when true peer exposure depends on tie strengths between ego and treated peers. Our method consistently outperforms all baselines because it can incorporate edge attributes and learn if those attributes are relevant for underlying peer exposure mechanisms.}}
    \label{fig:tie-str}
\end{figure*}
\begin{figure*}[!ht]
    \centering
    \includegraphics[width=\linewidth]{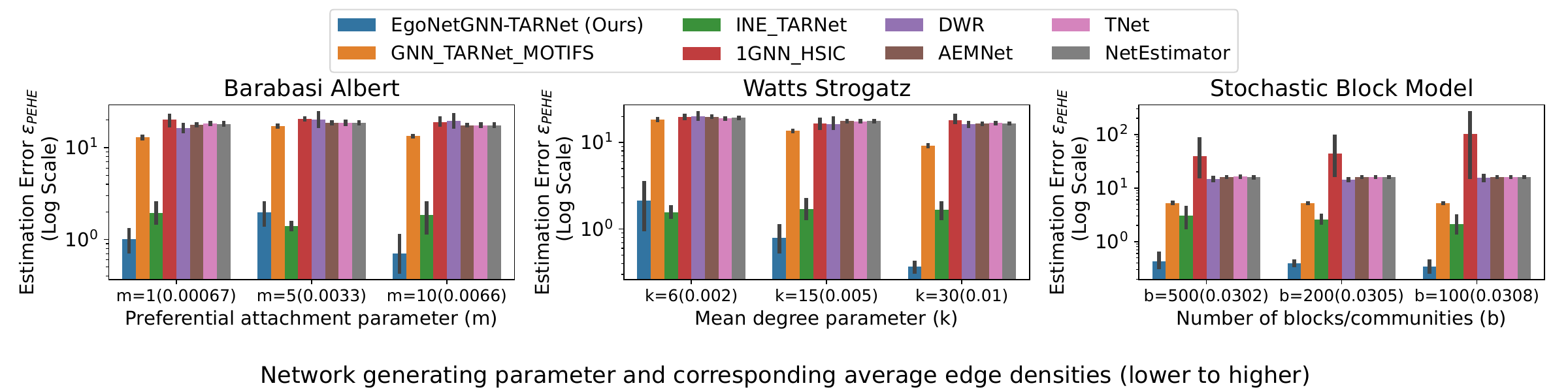}
    \caption{{Peer effect estimation error when true peer exposure depends on attribute similarity between ego and treated peers. Our method consistently outperforms all baselines because it can capture and learn if attribute similarity are relevant for underlying peer exposure mechanisms.}}
    \label{fig:attr-sim}
\end{figure*}
% \clearpage

% \begin{figure*}[!ht]
%     \centering
%     \includegraphics[width=\linewidth]{images/tie_strengths_kdd.jpg}
%     \caption{Peer effect estimation error when true peer exposure depends on tie strengths between treated peers. Our method is better than the motif-count based baseline for WS and BA networks but competitive to causal motif counts for the SB network.}
%     \label{fig:tie}
% \end{figure*}

\subsection{Semi-synthetic Data Experiments and Results}\label{ap-result-semi}

First, we present results for RQ2 for the Flickr dataset in Table \ref{tab:exp-semi-flickr}. Either \ourmodel-CFR or \ourmodel-TARNet is still the best performing model in all settings. For mechanisms involving attribute similarity and clustering coefficient, \ourmodel-TARNet is slightly better than \ourmodel-CFR, most likely due to \ourmodel-CFR's sensitivity to hyperparameter. INE-TARNet is the baseline with competitive performance.

\begin{table}[!h]
    \caption{Mean and standard deviation of peer effect estimation error ($\epsilon_{PEHE}$) for different methods in BlogCatalog (BC) dataset for four settings when true peer exposure mechanisms depend on clustering coefficients, connected components, mutual connections, and attribute similarity. Either \ourmodel-TARNet or \ourmodel-CFR outperforms all other baselines across multiple settings.}
    \label{tab:exp-semi-flickr}
    \centering
\resizebox{\textwidth}{!}{
\begin{tabular}{p{5.1em}p{3em}p{3em}p{3em}p{3em}p{3em}p{3.1em}p{3.1em}p{3em}p{3em}p{3em}}
\toprule
{Mechanisms} & Ours-TARNet & \newchange{Ours-CFR} & GNN-Motifs &        INE-TARNet &          1GNN-HSIC &                DWR &             AEMNet &                TNet &       NetEst & \newchange{CauGramer}\\
\midrule
Clus. Coef. &   {\small $\mathbf{4.93}_{\pm 1.6}$} & {\small $\underline{5.12}_{\pm 1.8}$} &   {\small $5.34_{\pm 1.5}$} &   {\small ${5.26}_{\pm 1.6}$} &   {\small $9.56_{\pm 4.9}$} &   {\small $9.51_{\pm 2.2}$} &    {\small $8.05_{\pm 5.5}$} &    {\small $9.75_{\pm 4.6}$} &   {\small $7.57_{\pm 1.3}$} & {\small $7.84_{\pm 0.7}$}\\
Con. Comp.    &   {\small $\underline{1.83}_{\pm 0.6}$} & {\small $\mathbf{1.40}_{\pm 0.5}$} &   {\small $2.80_{\pm 1.2}$} &   {\small ${1.85}_{\pm 0.7}$} &   {\small $3.36_{\pm 0.8}$} &   {\small $2.75_{\pm 0.6}$} &    {\small $4.69_{\pm 1.7}$} &    {\small $2.94_{\pm 0.9}$} &   {\small $2.67_{\pm 0.5}$} &   {\small $2.84_{\pm 0.6}$}\\
Mut. Con.      &   {\small ${2.38}_{\pm 1.3}$}  &   {\small $\mathbf{1.99}_{\pm 1.2}$}&    {\small $2.55_{\pm 0.5}$} &   {\small $\underline{2.36}_{\pm 0.6}$} &   {\small $4.03_{\pm 1.6}$} &   {\small $3.57_{\pm 1.7}$} &  {\small $10.95_{\pm 12.3}$} &  {\small $10.96_{\pm 17.2}$} &   {\small $4.24_{\pm 1.8}$} &   {\small $4.34_{\pm 1.9}$}\\
Attr. Sim.    &  {\small $\mathbf{11.32}_{\pm 6.6}$}&  {\small ${13.06}_{\pm 12.7}$} &  {\small $13.15_{\pm 10.8}$} &  {\small $\underline{12.17}_{\pm 8.8}$} &  {\small $16.94_{\pm 8.1}$} &  {\small $18.03_{\pm 9.7}$} &  {\small $17.43_{\pm 10.0}$} &  {\small $23.09_{\pm 20.3}$} &  {\small $16.87_{\pm 7.8}$} &  {\small $20.38_{\pm 11.6}$}\\
\bottomrule
\end{tabular}
}
\end{table}
Next, we utilize \ourmodel's feature mapping MPGNN $\hat{\phi}_f$ and outcome prediction model $\hat{f}_Y$ in the leading two baselines: GNN-TARNet-MOTIFS and INE-TARNet. The goal of this experiment is to ascertain the contribution of \ourmodel-TARNet's exposure mapping function $\hat{\phi}_e$. Table \ref{tab:bc-fm-om} shows the mean and standard deviation of peer effect estimation error ($\epsilon_{PEHE}$) for \ourmodel and these baselines in BlogCatalog (BC) dataset for four settings when true peer exposure mechanisms depend on clustering coefficients, connected components, mutual connections, and attribute similarity. The results show our method still performs better than the baselines, verifying the contribution of the learned exposure mapping function.

\begin{table}[h]
    \centering
    \caption{Mean and standard deviation of peer effect estimation error ($\epsilon_{PEHE}$) for \ourmodel and top baselines using \ourmodel's feature mapping and outcome prediction in BlogCatalog (BC) dataset for four settings when true peer exposure mechanisms depend on clustering coefficients, connected components, mutual connections, and attribute similarity.}
    \begin{tabular}{llll}
\toprule
Method &           EgoNetGNN-TARNet &          GNN-TARNet-MOTIFS &                 INE-TARNet \\
Mechanism              &                            &                            &                            \\
\midrule

Clustering Coefficient &  {\small $\mathbf{1.59}_{\pm 0.4}$} &  {\small $2.09_{\pm 1.2}$} &  {\small $2.73_{\pm 0.6}$} \\
Connected Components   &  {\small $\mathbf{2.98}_{\pm 0.8}$} &  {\small $4.08_{\pm 1.0}$} &  {\small $4.52_{\pm 1.0}$} \\
Mutual Connections     &  {\small $\mathbf{2.90}_{\pm 1.1}$} &  {\small $3.50_{\pm 0.7}$} &  {\small $4.66_{\pm 2.1}$} \\
Attribute Similarity   &  {\small $\mathbf{5.65}_{\pm 0.7}$} &  {\small $6.95_{\pm 0.9}$} &  {\small $5.86_{\pm 2.1}$} \\
\bottomrule
\end{tabular}
    \label{tab:bc-fm-om}
\end{table}

\subsection{Ablation Studies and Hyperparameter Sensitivity} \label{ap-result-ablation}
Table \ref{tab:balance} presents the performance of \ourmodel without balance loss, i.e., $\lambda_{bal}=0$, and with two different coefficients of balance loss, i.e., $\lambda_{bal}=0.01$ and $\lambda_{bal}=0.1$ for four settings when true peer exposure mechanisms depend on clustering coefficients, connected components, mutual connections, and attribute similarity. In general, using balance loss with a small coefficient results in a more robust performance. \ourmodel performs well for more complex peer influence mechanisms in the absence of balance loss. However, the performance for other mechanisms is comparatively poor in the absence of balance loss.

Table \ref{tab:outdim} shows the performance of \ourmodel for different output dimension of peer exposure embedding $\bm{\rho}_i$ for four settings when true peer exposure mechanisms depend on clustering coefficients, connected components, mutual connections, and attribute similarity.
As seen in the results, lower-dimensional peer exposure embeddings could lose expressiveness, while higher dimensions could introduce variance due to irrelevant contexts or violations of positivity. Lower-dimensional peer exposure embedding has better performance for simpler peer exposure mechanism like clustering coefficient and higher-dimensional peer exposure embedding has better performance for complex peer exposure mechanism like attribute similarity.

Table \ref{tab:noisy} shows the performance of EgoNetGNN and top baselines in the BlogCatalog Data when the network is augmented to make it noisy by randomly removing or adding 10 and 20 percent of edges. We expect the models to perform inconsistently or worse with higher noise. The results show that for different noisy settings, our model is consistently better than the baselines. The results, however, do not show an obvious trend of higher degradation in performance with high noise. This may be because the augmentation by randomly adding or removing edges may still preserve the signal to capture underlying peer exposure mechanisms.

\begin{table}[]
    \centering
    \caption{Performance of \ourmodel in BlogCatalog Data for different coefficients of balance loss for four settings when true peer exposure mechanisms depend on clustering coefficients, connected components, mutual connections, and attribute similarity.}
    \label{tab:balance}
    \begin{tabular}{llll}
\toprule
$\lambda_{bal}$ &                       $0.00$ &                       $0.01$ &                       $0.10$ \\
Mechanism              &                            &                            &                            \\
\midrule
Clustering Coefficient &  {\small $1.96_{\pm 1.1}$} &  {\small $\underline{1.59}_{\pm 0.4}$} &  {\small $\mathbf{1.33}_{\pm 0.3}$} \\
Connected Components    &  {\small $\mathbf{2.90}_{\pm 0.8}$} &  {\small $\underline{2.98}_{\pm 0.8}$} &  {\small $3.08_{\pm 1.0}$} \\
Mutual Connections      &  {\small $3.35_{\pm 0.7}$} &  {\small $\mathbf{2.90}_{\pm 1.1}$} &  {\small $\underline{2.92}_{\pm 0.7}$} \\
Attribute Similarity   &  {\small $\mathbf{5.54}_{\pm 0.6}$} &  {\small $\underline{5.65}_{\pm 0.7}$} &  {\small $5.77_{\pm 0.6}$} \\
\bottomrule
\end{tabular}
\end{table}

\begin{table}[t]
    \centering
     \caption{Performance of \ourmodel in BlogCatalog Data for different output dimension of peer exposure embedding $\bm{\rho}_i$ for four settings when true peer exposure mechanisms depend on clustering coefficients, connected components, mutual connections, and attribute similarity.}
    \label{tab:outdim}
    \begin{tabular}{llll}
\toprule
Output Dimension &                          2 &                          6 &                          10 \\
Mechanism              &                            &                            &                            \\
\midrule

Clustering Coefficient &  {\small $\mathbf{1.51}_{\pm 0.5}$} &  {\small $\underline{1.59}_{\pm 0.4}$} &  {\small $1.95_{\pm 1.4}$} \\
Connected Components   &  {\small $3.35_{\pm 0.4}$} &  {\small $\mathbf{2.98}_{\pm 0.8}$} &  {\small $\underline{3.03}_{\pm 0.8}$} \\
Mutual Connections     &  {\small $\underline{3.09}_{\pm 1.1}$} &  {\small $\mathbf{2.90}_{\pm 1.1}$} &  {\small $3.17_{\pm 0.9}$} \\
Attribute Similarity   &  {\small $6.55_{\pm 1.8}$} &  {\small $\underline{5.65}_{\pm 0.7}$} &  {\small $\mathbf{5.54}_{\pm 0.7}$} \\
\bottomrule
\end{tabular}
   
\end{table}

\begin{table}[ht]
    \centering
    \caption{Performance of \ourmodel and top baselines in the BlogCatalog Data when the network is augmented to make it noisy by randomly removing or adding a certain percentage of edges.}
    \label{tab:noisy}
    \begin{tabular}{lllllll}
\toprule
                   & Edge Augmentation &                        -20\% &                        -10\% &                         0\%  &                         10\% &                         20\% \\
Mechanism & Estimator &                            &                            &                            &                            &                            \\
\midrule
Attribute Similarity & EgoNetGNN-TARNet &  {\small $\mathbf{5.51}_{\pm 1.0}$} &  {\small $6.03_{\pm 0.9}$} &  {\small $\mathbf{5.65}_{\pm 0.7}$} &  {\small $\mathbf{5.67}_{\pm 0.5}$} &  {\small $\mathbf{5.77}_{\pm 0.5}$} \\
                   & GNN-TARNet-MOTIFS &  {\small $7.19_{\pm 2.1}$} &  {\small $6.74_{\pm 1.0}$} &  {\small $6.09_{\pm 0.2}$} &  {\small $6.66_{\pm 0.6}$} &  {\small $6.54_{\pm 1.1}$} \\
                   & INE-TARNet &  {\small $5.97_{\pm 1.2}$} &  {\small $\mathbf{5.88}_{\pm 1.6}$} &  {\small $6.01_{\pm 2.0}$} &  {\small $5.86_{\pm 0.6}$} &  {\small $6.50_{\pm 0.8}$} \\
\midrule                   
Clustering Coefficient & EgoNetGNN-TARNet &  {\small $\mathbf{1.55}_{\pm 0.4}$} &  {\small $\mathbf{1.59}_{\pm 0.3}$} &  {\small $\mathbf{1.59}_{\pm 0.4}$} &  {\small $\mathbf{1.36}_{\pm 0.4}$} &  {\small $2.14_{\pm 1.2}$} \\
                   & GNN-TARNet-MOTIFS &  {\small $2.19_{\pm 1.1}$} &  {\small $1.89_{\pm 0.6}$} &  {\small $2.04_{\pm 0.7}$} &  {\small $1.90_{\pm 0.5}$} &  {\small $\mathbf{1.94}_{\pm 0.7}$} \\
                   & INE-TARNet &  {\small $1.79_{\pm 0.4}$} &  {\small $1.80_{\pm 0.4}$} &  {\small $2.24_{\pm 0.6}$} &  {\small $1.78_{\pm 0.4}$} &  {\small $1.95_{\pm 0.5}$} \\
\midrule                  
Mutual Connections & EgoNetGNN-TARNet &  {\small $\mathbf{3.09}_{\pm 0.3}$} &  {\small $\mathbf{3.00}_{\pm 0.5}$} &  {\small $\mathbf{2.90}_{\pm 1.1}$} &  {\small $\mathbf{3.32}_{\pm 0.9}$} &  {\small $\mathbf{2.85}_{\pm 0.5}$} \\
                   & GNN-TARNet-MOTIFS &  {\small $4.00_{\pm 1.2}$} &  {\small $3.67_{\pm 0.6}$} &  {\small $3.83_{\pm 0.7}$} &  {\small $4.57_{\pm 2.3}$} &  {\small $4.23_{\pm 2.5}$} \\
                   & INE-TARNet &  {\small $3.41_{\pm 0.6}$} &  {\small $3.08_{\pm 0.7}$} &  {\small $4.58_{\pm 2.0}$} &  {\small $3.61_{\pm 1.5}$} &  {\small $3.60_{\pm 1.2}$} \\
\bottomrule
\end{tabular}

\end{table}

\end{document}